\documentclass[twoside]{article}

\usepackage{color-edits}
\addauthor{vs}{red}
\addauthor{eb}{blue}

\usepackage{diagbox, xcolor}
\usepackage[round]{natbib}

\usepackage{amsthm}
\usepackage{booktabs}
\newtheorem{assumption}{Assumption}
\usepackage{algorithm}
\usepackage{algorithmic}
\usepackage{graphicx}
\usepackage{amsmath}
\usepackage{amssymb}

\usepackage{amsmath,amsfonts,bm}

\def\eqref#1{equation~\ref{#1}}
\def\1{\bm{1}}

\DeclareMathAlphabet{\mathsfit}{\encodingdefault}{\sfdefault}{m}{sl}
\SetMathAlphabet{\mathsfit}{bold}{\encodingdefault}{\sfdefault}{bx}{n}

\usepackage{booktabs}
\newtheorem{theorem}{Theorem}
\newtheorem{corollary}{Corollary}

\usepackage[colorlinks=true,citecolor=blue,linkcolor=blue]{hyperref}

\newlength\titlebox 
\begin{document}

% Content before the two-column layout starts
\twocolumn[
\begin{center}
    \LARGE \bfseries Predicting Long Term Sequential Policy Value Using Softer Surrogates \\
    \vspace{1em}
    \large Hyunji Nam, Allen Nie, Ge Gao, Vasilis Syrgkanis, Emma Brunskill \\
    \vspace{1em}
    \large \normalfont Stanford University
    \vspace{1em}
\end{center}

]
 \begin{abstract}
Off-policy policy evaluation (OPE) estimates the outcome of a new policy using historical data collected from a different policy. However, existing OPE methods cannot handle cases when the new policy introduces novel actions. This issue commonly occurs in real-world domains, like healthcare, as new drugs and treatments are continuously developed. Novel actions necessitate on-policy data collection, which can be burdensome and expensive if the outcome of interest takes a substantial amount of time to observe--for example, in multi-year clinical trials. This raises a key question of how to predict the long-term outcome of a policy  after only observing its short-term effects? Though in general  this problem is intractable, under some surrogacy conditions, the short-term on-policy data can be combined with the long-term historical data to make accurate predictions about the new policy's long-term value. In two simulated healthcare examples--HIV and sepsis management--we show that our estimators can provide accurate predictions about the policy value only after observing 10\% of the full horizon data. We also provide finite sample analysis of our doubly robust estimators.
\end{abstract}

\section{INTRODUCTION}
Policy evaluation of long-term outcomes under a new target policy using historical data collected under a different behavior policy is broadly investigated in offline reinforcement learning ~\citep{dudik2011doublyrobustpolicyevaluation, thomas2016dataefficient,sergey2020}. Off-policy policy evaluation (OPE) methods rely on shared coverage of the state and action spaces taken by a behavioral policy and a new target policy. However, in real-world practices, with rapid technological advances, new actions are constantly being developed and tested to enhance the efficacy of policies. When the coverage assumption is violated due to the introduction of novel actions, existing methods cannot be directly applied and provide reliable estimates of the target policy's value. This requires testing the new policy, which generally take a long horizon to validate, such as several years for a new drug in clinical trials. 

\begin{figure}
\begin{center}
\centerline{\includegraphics[width=\columnwidth]{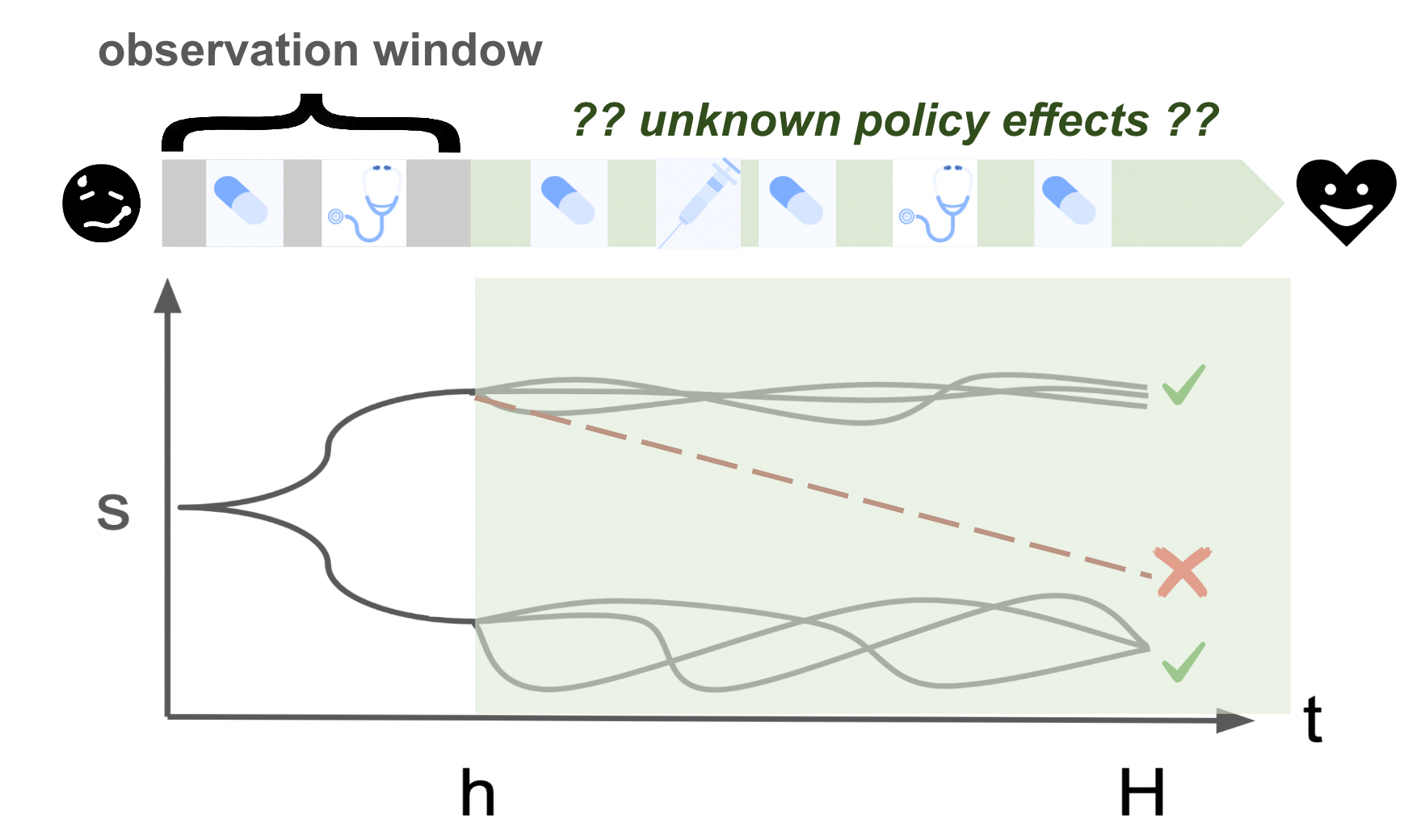}}
\label{fig:diagram}
\caption{We are interested in the task of predicting the patient's long-term outcome only after a short observation window. The standard surrogacy assumption ~\cite{athey2019} fails due to the red trajectory where later decisions cause a very different outcome than other trajectories with the same initial $h$ observations. In this work we leverage a narrower definition of surrogacy similar to ~\citet{battochi2021}, and show it enables us to perform effective policy evaluation.} % by carefully characterizing the behavioral policy and the target policy's relationship,
\end{center}
\end{figure}

% %receives a suboptimal outcome unlike other trajectories with the same initial observations--for example, if the patient stops taking pills and falls back to unstable conditions after the initial period. To avoid such adversarial cases, we relax the existing assumption to 

Naively deploying the new policy to observe its long-term outcome can be burdensome and expensive. Especially in high-stake domains like healthcare, the outcome of interest takes a substantial amount of time to observe. Determining the policy's value based only on its short-term effect may also be misleading; for example, patients may respond well to aggressive cancer treatments in the short term, but the long-term health effects may be suboptimal, and even harmful~\cite{MORGAN199865}. Many other domains, such as education and e-commerce, also face similar challenges. For example, it is common to measure the impact of using an intelligent tutoring system over the entire school year by the end-of-year assessments ~\citep{zheng2019using},
and the impact of ads or promotions over a multi-month customer churn or engagement ~\citep{zhang2023evaluating}. While the new policy's long-term outcome is unknown or expensive to observe, one may be able to estimate the long-term outcome using some available data collected under a different policy with both short- and long-term outcomes. Naturally, this has led to the interest in estimating the long-term outcome of a target policy using its short-term data and full-horizon historical data of a different policy and prior works have achieved inspiring results under specific assumptions~\citep{saito2024long,cheng2020,tang2022reinforcement,zhang2023evaluating,tran2024icml,shi2023dynamic}

Without further assumption, this problem is intractable to solve for sequential policies taking different actions over a long horizon. As a simple adversarial example of why this is infeasible to solve without structural assumptions, consider two deterministic policies that go through the same chain of states except at the last time-step. One policy leads to the optimal last state with a reward of 100, and the other policy leads to the suboptimal state with a zero reward. Since both policies have identical effects on the states until the last step, they require observing for the full-horizon in order to be distinguished from each other.

One possible structural assumption to make is surrogate index, which is well studied in causal inference and econometrics~\citep{athey2019, athey2024estimatingtreatmenteffectsusing}. In this framework, we assume ``surrogates” observed from a short horizon make the long-term outcome independent of future actions. This is a very strong assumption to make in sequential decision-making settings that are our primary focus. In the above adversarial example, surrogacy only holds with the full horizon. In an effort to relax this assumption, \citet{battochi2021} introduces a ``dynamic invariance" assumption. They define a behavioral policy and a new target policy, and require that the behavioral policy’s estimand and the target policy’s estimand have the same relationship with the surrogates. We make a similar assumption but for the purpose of estimating the long-term value of a sequential policy rather than the delayed effect of a single-step intervention.

We propose two estimators to tackle the new empirical problem, i.e., estimating values of long-term sequential policies with short-term data: one based on weighted regression and the other using doubly robust methods to estimate the long-term policy value using the short-term data and the historical data from a behavioral policy. Our estimators operate under the assumption that conditioned on the short-horizon state trajectories, the expected future returns of the new policy matches the expected future returns of the behavioral policy. We call the short-horizon state trajectories ``softer surrogates." While we will shortly discuss this in greater depth, we briefly give some motivating scenarios where this is plausible, which are visualized in Figure \ref{fig:diagram}.

Consider two treatment regimes: with the first treatment, half of the patient population responds well to the treatment (top trajectories in Fig. \ref{fig:diagram}) and the other half struggles to recover (bottom trajectories). A new drug is introduced which works well for 20\% of the previously struggling patients. These treatments lead to different distributions of patient outcomes--first treatment works for 50\% and the second treatment works for 70\%--but importantly observe that if the patients are improved after the initial 5 days of the treatment, they will continue to observe positive effects as long as they are continued to be treated. In other words, given the first 5 days observations, both treatment regimes have the same expected future returns. However, the standard surrogacy assumption would require an even stronger condition that the patients who are recovering after the 5 days will experience the same future returns even if they stop their treatment after the initial period. This example highlights that our proposed assumption is much more realistic than the standard surrogate assumption. 

Even when the soft surrogate condition does not hold, and therefore, our estimators may have some bias, our empirical results on the simulators of HIV treatment~\citep{ernst} and sepsis management~\citep{oberst2019counterfactual} show that our methods can still produce accurate estimates of the policy's long-term value. While one might expect this to work only when the soft surrogates are constructed from relatively long-horizon observations, our experiments show that even with only 10\% of the full horizon data, our estimators can accurately predict whether the new policy's expected returns will outperform the behavioral policy's with significant p-values $< 10^{-6}$ and achieve lower mean-squared-errors compared to existing baselines.

Our work tackles the important problem of estimating long term policy value with novel actions, which cannot be handled by existing OPE methods and is expensive to solve with Monte Carlo sampling. We propose estimators based on soft surrogates constructed from short-horizon data of the new policy and leverage long-horizon data of the behavioral policy in learning these estimators by relaxing the standard surrogacy assumption. We provide finite sample analysis of our estimators based on doubly robust methods~\cite{chernozhukov2017doubledebiasedmachinelearningtreatment}, which may be of independent interest, and empirical results in clinical simulators of HIV and sepsis treatment, showcasing the promising application of our methods to realistic domains.
\section{RELATED WORK}
A powerful approach to predicting the long-term outcomes of novel interventions comes from seminal work on surrogacy~\citep{prentice1989surrogate,athey2019,saito2024long}. The surrogacy assumption (or surrogate index) is when one or more intermediate variables render a past intervention independent of the delayed (future) outcome. 
In such work, historical data is used to learn estimators for predicting the long-term outcome given the short-term surrogate observations, and these estimators are then applied to new short-term observations for a new intervention. These ideas have been more recently used with machine learning methods~\citep{cheng2020,zhang2023evaluating}.
%and some work has highlighted the benefit of combining short-term and long-term signals for offline RL~\citep{minmin_user_surrogate}. 
%Interestingly, in low signal-to-ratio settings (for example, user's long-term visiting frequency to a platform in recommendation systems as the noisy long term target in \cite{minmin_user_surrogate}), even when the long term outcome is observed after a full horizon, using short-term measurements can be more efficient for long term policy evaluation. 
Other papers have improved the surrogate methods to be more robust even when surrogacy does not strictly hold~\citep{saito2024long, kallus2024rolesurrogatesefficientestimation}, and some work has leveraged surrogates while actively exploring in a delayed multi-armed bandit setting~\citep{grover2018best,mcdonald2023impatient}. Almost all such work has focused on when there is a single intervention, followed by surrogate measures, and then a delayed outcome.  In contrast, our aim is to estimate the value of a sequential RL policy that will continue to influence later states and rewards after the short observation window.%, and we do not assume the short-term trajectory is a sufficient surrogate that renders the downstream states nor rewards independent of the target policy. 

To our knowledge, there has been very little work on the setting we consider of combining novel actions and long-term policy value estimation. Empirically, \citet{severson} showed that domain-specific feature engineering can be done to predict battery lifetime given short-term data, and \citet{attia_battery} built on this work by applying Bayesian optimization to quickly learn  good charging protocols. Likely most similar to our work is ~\citet{battochi2021} which used surrogate indices from short-term data to estimate the impact of a novel action. Their focus is on estimating the long-term performance of a single action followed by no further interventions. In contrast,  we are interested in RL policies which continue to take different actions in future horizons, rather than having the effects of the old actions persist over time. 

Ultimately, we show that this problem can be framed as inference on a linear functional of a regression function under a covariate shift. From this perspective it falls in the general setting analyzed in \citet{chernozhukov2023automatic}. We show how the techniques in \citet{chernozhukov2023automatic} can be utilized to provide unbiased and doubly robust estimators that allow for the construction of asymptotically valid confidence intervals, under assumptions on the achievable estimation rates for nuisance parameters (e.g. value function, density ratio). One technical contribution of our work is to provide finite sample high-probability analogues of the asymptotic results given by  \citet{chernozhukov2023automatic} in the case of covariate shift. Our results provide estimation rates with exponential tails, which is not covered by prior work with finite sample analysis ~\cite{chernozhukov2022simplegeneraldebiasedmachine}, and for any estimation rate regime of the nuisance components, which can be of independent interest.

\section{Problem Setting \& Notation}
We consider an $H$-step sequential decision process consisting of states $\mathcal S$ and actions $\mathcal A$, and use $t$ to denote the time-step. We allow for non-Markov decision processes where the transition dynamics is defined as a function of all past states and actions, $P(s_{t+1} | s_0, a_0, ..., s_t, a_t)$ as the Markov assumption is not required by our estimators. We define the reward as a function of states $R(s)$, and the returns $G$ as the sum of per-step rewards over the full $H$-step horizon. We define the trajectory \textit{only as a sequence of states and rewards} (excluding the actions), $\tau_{0:t} = (s_0, r_0, ..., s_t, r_t)$, and similarly, $\tau_{h:H}$ denotes the partial trajectory from $h$ to $H$. Policy is a mapping from states to actions: $\pi: \mathcal S \mapsto \mathcal A$, and the corresponding policy value is the expected returns of executing this policy from some initial state distribution $d_0: \mathcal S \mapsto [0, 1]$: $V^\pi = \mathbb E_{\tau_{0:H} \sim \pi, d_0}[\sum_{t=0}^H r_t]$. In order to incorporate novel actions, we additionally define $\mathcal A^+$, which need not overlap with $\mathcal A$. 

We assume the historical data $\mathcal D_b$ is collected under a behavioral policy $\pi_b$, which consists of $(\tau_{0:h}, \tau_{h:H}, G)$ representing the short-horizon trajectory, the long-horizon trajectory, and the returns. Now, we are given a new policy to evaluate (target policy): $\pi_e: \mathcal S \mapsto \mathcal A^+$, and the goal is to predict the new policy's value $V^{\pi_e}$ under the same initial state distribution only using the short-term data of the new policy. We call this short-term on-policy data $\mathcal D_e$.

\section{Assumption}
Before introducing our estimators, we first describe the assumptions that will enable the estimation of the long-term policy value only using short-horizon data even when the action spaces differ between the new policy and the behavioral policy. 

\begin{assumption}
The reward is only a function of states, or a history of states.
\end{assumption}
Especially in healthcare domains ~\cite{oberst2019counterfactual, hiv_states}, where states represent the patient's medical conditions, rewards are defined in terms of the patient's states. Similarly in education, states may represent student's knowledge~\cite{Mu_Jetten_Brunskill_2020}, where it's natural to define reward based on student's knowledge acquisition.

\begin{assumption} (Soft surrogacy)\label{assmp:surrogacy} Let $\tau_{0:h}$ be the initial $h$-step trajectory, and $\tau_{h:H}$ be the remaining trajectory. The target policy $\pi_e$ and the historical policy $\pi_b$ satisfy the following relationship with $\tau_{0:h}$: 
\begin{equation}
\mathbb E_{\tau_{h:H} \sim \pi_e} \left[ \sum_{t=h}^H r_t | \tau_{0:h} \right] = \mathbb E_{\tau_{h:H} \sim \pi_b} \left[\sum_{t=h}^H r_t | \tau_{0:h} \right]
\end{equation}
\end{assumption}
Conditioned on the initial sequence $\tau_{0:h}$, the expected future returns under the new policy is the same as the expected future returns under the behavioral policy. Similar condition has been considered by prior work ~\cite{battochi2021}. It's important to highlight that the future trajectories $\tau_{h:H}$ may differ between $\pi_b$ and $\pi_e$ as long as their sum of rewards matches. One sufficient, but not necessary, condition of Assumption 4.2 is if $V_{H-h}^{\pi_b}(s_h) = V_{H-h}^{\pi_e}(s_h)$.

To show that this is a relaxed version of the surrogacy assumption, we compare Assumption 4.2 to the standard surrogacy assumption if it were applied to sequential RL policies, which would be $\forall \pi, \pi' \in \Pi$
\begin{equation}
\mathbb E_{\tau_{h:H} \sim \pi} \left[\sum_{t=h}^H r_t | \tau_{0:h}\right] = \mathbb E_{\tau_{h:H} \sim \pi'} \left[\sum_{t=h}^H r_t | \tau_{0:h}\right]
\end{equation} We make an important observation that with the large action spaces and flexible policy classes, the surrogacy assumption in Equation (2) will not hold in most RL settings. However, we can relax this assumption to focus on a specific target policy and a behavioral policy rather than any realizable policy in the given policy class. While our theory relies on Assumption 4.2, our empirical results with the clinical simulators of HIV and Sepsis management suggest that even when the soft surrogacy assumption is violated, our estimators can give reliable estimates of the long-term policy value with short-term data. In non-healthcare settings, ~\citet{severson, attia_battery} have considered a similar empirical condition to build a prediction model of the battery life time given the short-term discharge data.

\begin{assumption}\label{assmp:coverage}
The behavioral policy $\pi_b$ covers the distribution of short-horizon trajectories generated under the new target policy $\pi_e$ such that
\begin{equation}
\frac{p(\tau_{0:h}|\pi_e)}{p(\tau_{0:h}|\pi_b)} > 0 \quad  \forall p(\tau_{0:h} | \pi_e) > 0
\end{equation}
\end{assumption}
This is analogous to the state and action coverage in offline RL, which requires $\pi_b(.|s) > 0 \; \forall \pi_e(.|s) > 0$~\cite{liu2022offlinepolicyoptimizationeligible}. 
% To allow the actions to differ, we instead require that the generated trajectories $\tau_{0:h}$ under the new target policy are in support of the distribution generated by the behavioral policy.

% \an{Feedback was to be up front about the assumption, but also clarify that our assumption relaxes the stricter assumption -- otherwise the problem is infeasible to solve}

\section{Estimators}
We are interested in estimating values of long-term sequential policies with short-term on-policy data and long-term off-policy data. Algorithm \ref{alg:main_alg} shows how we leverage the historical data along with the limited on-policy samples to handle estimation. We propose two kinds of regression models: unweighted and weighted, and two policy value estimation methods: one directly from the regression outputs and the other based on double machine learning~\cite{chernozhukov2017doubledebiasedmachinelearningtreatment}, and study their theoretical (Section \ref{section:theory}) and empirical properties (Section \ref{section:experiments}).

\begin{algorithm}[tb]
   \caption{Long-term policy value estimation with soft surrogates}
   \label{alg:main_alg}
\begin{algorithmic}
   \STATE {\bfseries Input:} long-term historical data $\mathcal D_b$, new target policy $\pi_e$, $\mathcal D_e = \emptyset$, experiment budget $B$ 
    \FOR{$i=1$ {\bfseries to} $B$}
   \STATE Run $\pi_e$ for short-horizon of $h$ steps.
   \STATE Collect $\tau_{0:h} \sim \pi_e$ into $\mathcal D_e$.
   \ENDFOR
   \STATE Train a regression model $\hat f$ using $\mathcal D_b$ (Eq. 4, 6).
   \STATE Estimate the long-term policy value $\hat V^{\pi_e}$(Eq. 5, 7-8).
   \STATE {\bfseries Return:} $\hat V^{\pi_e}$
\end{algorithmic}
\end{algorithm}

\subsection{Soft surrogate estimator}
First, we propose a regression model that takes in the short-horizon trajectory as inputs and predicts the full-horizon outcome. Specifically, the regressor $f$ is learned to predict $G$ given $\tau_{0:h}$ in the historical data $\mathcal D_b$. \begin{equation}
f_{\text{soft}} = \arg \min_f \mathbb E_{(\tau_{0:h}, G) \sim \mathcal D_b} [f(\tau_{0:h}) - G]^2.
\end{equation} The long-term value of a new policy $\pi_e$ is computed by averaging $f_{\text{soft}}$ across the short-horizon trajectories in $D_e$: 
\begin{equation}
\hat V^{\pi_e}_{\text{soft}} = \frac{1}{|D_e
|} \sum_{\tau_{0:h} \in \mathcal D_e}f_{\text{soft}}(\tau_{0:h}).
\end{equation} While $f_{\text{soft}}$ can be learned directly from $\mathcal D_b$, it may be helpful to re-weight the loss to prioritize accurate predictions over the short-horizon trajectories that are most likely to occur under the target policy. To accomplish this, we use weighted regression methods ~\cite{sugiyama_breg, sugiyama_cv}, which first estimate the density ratio of the trajectories under the target policy and the behavioral policy, then re-weight the losses according to the estimated density ratios. We follow the prior work by ~\citet{sugiyama_breg} in formulating the problem of density ratio estimation as learning a classifier $h(\tau_{0:h})$ to predict if the sample trajectory came from the behavioral policy or the target policy. This yields the weighted regression model 
\begin{equation}
f_{\text{w-soft}} = \arg \min_f \mathbb E_{(\tau_{0:h}, G) \sim D_b} \left[\hat h(\tau_{0:h}) \left(f(\tau_{0:h}) - G \right)^2 \right],
\end{equation}where $\hat h(\tau_{0:h}) = \frac{p(\tau_{0:h} | \pi_e)}{p(\tau_{0:h} | \pi_b)}$. To predict the long-term policy value, we replace $f_{\text{soft}}$ in equation (4) with $f_{\text{w-soft}}$ and call this estimator $\hat V_{\text{w-soft}}^{\pi_e}$. Our empirical results compare the weighted and the unweighted regression estimates.

\subsection{Doubly robust soft surrogate estimator} We present an estimator based on double machine learning ~\citep{chernozhukov2023automatic}, or de-biased machine learning ~\citep{battochi2021, kallus2024rolesurrogatesefficientestimation} to improve the estimator presented in Section 5.1 to be robust to potential errors in the regression model or the density ratio estimator.  

Following $k$-fold cross fitting ~\citep{chernozhukov2023automatic}, we split the dataset $\mathcal D_b$ into K folds of equal size. Let $\mathcal D^{(k)}_b$ be the samples in fold $k$ and $\bar {\mathcal D}^{(k)}_b$ be its complement. Similarly, split the dataset $\mathcal D_e$ into K folds, and use the same notations $\mathcal D_e^{(k)}$ and $\bar{\mathcal D}_e^{(k)}$ to denote the $k$-th fold and its complement. Let $\hat h^{(k)}$ and $\hat f^{(k)}$ be the density ratio estimator and the regression model trained on $\bar {\mathcal D}^{(k)}$.  The $k$-th cross fit estimate is 
\begin{eqnarray}
\hat V^{(k)}\!\!\!\! &=&\!\!\!\! \frac{1}{|\mathcal D^{(k)}_b|} \sum_{(\tau_{0:h}, G) \in \mathcal D^{(k)}_b} \hat h^{(k)}(\tau_{0:h}) \left( G - \hat f^{(k)}(\tau_{0:h}) \right)  \nonumber \\
\!\!\!\! &&\!\!\!\! + \frac{1}{|\mathcal D_e^{(k)}|} \sum_{\tau_{0:h} \in \mathcal D_e^{(k)}} \hat f^{(k)} (\tau_{0:h}),
\end{eqnarray} where the first sum is over a subset of the historical data gathered by $\pi_b$, and the second term is over a subset of the short-horizon on-policy data sampled from the target policy $\pi_e$. The second term serves as a baseline and is a slight departure from the standard double machine learning, since here we have access to limited on-policy data.

The final estimate of the  long-term policy value is the average across all cross-fit estimates: 
\begin{equation}
\hat V^{\pi_e}_{\text{dr}} = \frac{1}{K} \sum_{i=1}^k \hat V^{(k)}
\end{equation}
Similarly as before, $\hat f^{(k)}$ in Equation (7) can be either the unweighted (Equation 6) or the weighted (Equation 4) regression model. Depending on the choice between Equation (5) and Equation (8) as well as the regression model, there are 4 possible estimators for the target policy--namely, (i) soft surrogate estimator $\hat V_{\text{soft}}$ (unweighted by default), (ii) weighted soft surrogate estimator $\hat V_{\text{W-soft}}$, (iii) doubly robust estimator $\hat V_{\text{dr-soft}}$, and (iv) doubly robust weighted estimator $\hat V_{\text{dr-w-soft}}$. Our toy example compares the robustness of these estimators to model mis-specification, while our clinical simulation results demonstrate their usefulness in realistic scenarios even with possible violations of Assumption 4.2.
\section{THEORY}\label{section:theory}
We focus on the finite sample analysis of $\hat V^{\pi_e}_{\text{dr}}$, which may be of independent interest. Closest prior work~\cite{chernozhukov2022simplegeneraldebiasedmachine} also gives a finite sample bound, but our result shows exponential tails in the covariate shift settings. Our analysis depends on the following $L_2$ errors between the estimated nuisance parameters ($\hat f, \hat h$) and the true nuisance parameters ($f, h$). For notational simplicity, we drop the subscript from the short-horizon trajectory $\tau_{0:h}$. For detailed exposition, review Appendix \ref{appendix:theory}.

\begin{equation}
\epsilon_e = \sqrt{\mathbb E_{\tau \sim \pi_e} \left[\left(\hat f^{(k)}(\tau) - f(\tau) \right)^2 \right]}
\end{equation}

\begin{equation}
\epsilon_b = \sqrt{\mathbb E_{\tau \sim \pi_b} \left[\left(\hat f^{(k)}(\tau) - f(\tau) \right)^2 \right]}
\end{equation}

\begin{equation}
\epsilon_h = \sqrt{\mathbb E_{\tau \sim \pi_b} \left[\left(\hat h^{(k)}(\tau) - h(\tau) \right)^2 \right]}
\end{equation}

Note that we are defining these errors with respect to the two different data distributions: Expression (9) depends on the short-horizon on-policy data, while (10) and (11) depend on the historical data distribution.

\begin{theorem}[Variance-Based Rate for DR]\label{thm:variance-dr}

   Assume that $|G|, |f^{(k)}(\tau)|, |f(\tau)|$ are asymptotically bounded by $C_1H$ and $\hat{h}^{(k)}(\tau), h(\tau)$ are asymptotically bounded by $C_2$, where $C_1$ and $C_2$ are constants.
   Under Assumptions \ref{assmp:surrogacy} and \ref{assmp:coverage}, w.p. at least $1-\delta$:
    \begin{eqnarray*}
        |\hat V^{\pi_e}_{\text{dr}}-V^{\pi_e}| \leq \sqrt{\frac{2\text{Var}_{\tau_ \sim \pi_e}(f(\tau))\log(4K/\delta)}{|\mathcal D_e|}}   \\
         + \sqrt{\frac{2 \mathbb E_{\tau \sim  \pi_b}\left[h(\tau)^2\,(G - f(\tau))^2\right]\log(4K/\delta)}{|\mathcal D_b|}}    \\
         + \frac{1}{K}\sum_{k=1}^K  \epsilon_{b} \cdot \epsilon_{h} + \underbrace{\frac{2 H K\log(4K/\delta)}{|\mathcal D_e|} }_{\text{(4)}} \\
         + \underbrace{\max_{k}\epsilon_{e}  \sqrt{\frac{2K\log(4K/\delta)}{|\mathcal D_e|}}}_{\text{depends on consistency of $\hat f$ w.r.t. $\pi_e$}}
          + \underbrace{\frac{4C_1 C_2 H K\log(4K/\delta)}{
          |\mathcal D_b|
          }}_{\text{(6)}} 
         \\
         + \underbrace{3 C_1 H \!\max_{k}\left\{
    \epsilon_{b}  + \!\epsilon_{h}\right\}\!\!\sqrt{\frac{2K\log(4K/\delta)}{|\mathcal D_b|}}}_{\text{depends on consistency of $\hat f, \hat k$ w.r.t. $\pi_b$}} 
    \end{eqnarray*}
\end{theorem}

We are interested in understanding when the first two terms, which depend on the variance under the true, unknown nuisance parameters $f, h$, dominate the bound. Note that the terms (4) and (6) are dominated by the first two terms since these have a slower rate dependence on $|\mathcal D_b|$ and $|\mathcal D_e|$. As long as $\hat f$ is consistent for the target policy's data distribution, the fifth term is dominated by the first term. Similarly, if the nuisance parameters, $\hat f, \hat h$, are consistent under the behavioral policy's data distribution, the last term is dominated by the second term. Finally the third term is a product of the errors under the behavioral policy's data distribution. For this to be of lower order than the first two terms, it may be particularly important for $\hat f$ to behave well under the behavioral policy's distribution. This suggests that fitting $\hat f$ using unweighted regression in Equation (4) may be beneficial to the weighted regression. We include the full proof in \ref{appendix:dr_bias_bound} (bias) and \ref{thm:variance-dr} (variance).

\begin{corollary}
If $\hat f$ and $\hat h$ are asymptotically consistent at any rate, then $\hat V^{\pi_e}_{\text{dr}}$ is consistent. 
\end{corollary} This is a direct result of Theorem 6.1 as the error terms go to zero asymptotically (see proof in \ref{appendix:dr_consistency}). We also characterize the doubly robust property of $\hat{V}^{\pi_e}_{\text{dr}}$.

\begin{theorem}
Define the regression error as $\triangle(\tau, \hat f) = \hat f(\tau)-f(\tau)$ and the relative density ratio estimate as $\delta(\tau, \hat h) = \frac{\hat h(\tau)}{h(\tau)}$. Under Assumptions 4.2 and 4.3, the bias of $\hat V^{\pi_e}_{\text{dr}}$ is $\hat V^{\pi_e}_{\text{dr}} - V^{\pi_e}$
\begin{equation}
=\frac{1}{K}\sum_{k=1}^K \mathbb E_{\tau \sim \pi_e} \left[ \triangle(\tau, \hat f^{(k)}) \left(1 - \delta(\tau, \hat h^{(k)})\right)\right],
\end{equation} which is the average of the product-bias terms across the K folds. 
\end{theorem} Full proof is in \ref{eqn:slev_consistent}.

\begin{corollary}
If $\forall \tau$, either $\hat f(\tau) = f(\tau)$ or $\hat h(\tau) = h(\tau)$, then $\hat V_{\text{dr}}^{\pi_e}$ is unbiased. 
\end{corollary}
This is a direct result of Theorem 6.3 when either $\triangle (\tau, \hat f) = 0$ or $\delta(\tau, \hat h)$ = 0.

While $\hat V^{\pi_e}_{\text{soft}}$ is also a consistent estimator of $V^{\pi_e}$ if the regression model $f$ is asymptotically consistent (proof in \ref{theorem-productbias}), the doubly robust estimator can provide some robustness against regression bias. More broadly, it is often the case that doubly robust estimators achieve semi-parametric efficiency, but the regression estimators do not. 
\section{Experiments}\label{section:experiments}
First we show a toy example to demonstrate the robustness of the doubly robust estimator to either the regression or the density ratio estimator mis-specification that is consistent with our theoretical understanding. Then we conduct extensive experiments on two clinical domains, HIV treatment~\cite{ernst} and sepsis management in ICU~\cite{oberst2019counterfactual}. 

\subsection{The Robustness of the Proposed Estimator}
We consider a toy domain with scalar (continuous) states and the initial state distribution $d_0$ evenly spread across the interval $[0, 1.5]$. Transitions under the behavioral policy are defined as: 
\[
s_{t+1} =
\begin{cases} 
s_t & \text{w.p 0.5} \\
(-0.6 + 0.1  u) s_t & \text{w.p 0.45}, u \sim Unif[0, 1)  \\
1.5 & \text{otherwise}
\end{cases},
\]and transitions under the new target policy are \[
s_{t+1} =
\begin{cases} 
1.5 & \text{if $s_t < 1.25$} \\
0 & \text{otherwise}
\end{cases}.
\]Every state has an additive gaussian noise $\sim N(0, 0.1)$. We define a quadratic relationship between the surrogate $\tau_{0:1}$ observed after one step and the long-term return $G$ (for simplicity, we assume Assumption 4.2 holds with $h = 1$). The on-policy data $\mathcal D_e$ has 100 samples of $\tau_{0:h}$ and missing long-term returns, and the historical data $\mathcal D_b$ contains 5000 tuples of $(\tau_{0:h}, \tau_{h:H}, G)$. The goal is to predict the average long-term returns under the target policy $\pi_e$ by leveraging the historical data. We evaluate the weighted, unweighted, and doubly robust estimators in cases of regression or density ratio estimator mis-specification. 

In order to simulate regression model mis-specification, we restrict the regression class to be linear, i.e., $f_\theta(\tau_{0:1}) = \theta^\top [s_0, s_1]$, rather than the correctly specified quadratic form $f_{\theta^*}(\tau_{0:1}) = {\theta^*}^\top [s_0, s_1, s_1^2]$. While $f_\theta$ is incorrect for describing the behavior under $\pi_b$, it is still correct for $\pi_e$ which induces $s_1$ to always be either 0 or 1.5. We expect the unweighted regression to fail in this case because $f_\theta$ cannot fit $\mathcal D_b$, but the weighted regression, if the density ratio estimates are correct, should still be able to handle the model mis-specification by only fitting the points that are likely under the target policy's distribution. 

In order to bias the density ratio estimates, we added a non-zero Gaussian noise $\sim N(10, 10)$ to the denominator of $\frac{p(\tau_{0:h}|\pi_e)}{p(\tau_{0:h}|\pi_b)}$. We expect the doubly robust estimator to still be accurate even with the incorrect density ratio estimates if the regression model is correct.

\begin{table}[t]
\caption{MSE of estimators under regression or density ratio model mis-specification (mean $\pm$ std across 200 seeds) .}
\label{sample-table}
\vskip 0.15in
\begin{center}
\begin{small}
\resizebox{\linewidth}{!}{\begin{tabular}{lcccr}
\toprule
\textbf{Estimator} & \textbf{Both correct} & \textbf{Regressor} & \textbf{Density ratio} \\
\midrule
Unweighted    & 0.002 (0.003) & \textbf{0.914 (0.064)} & 0.002 (0.003) \\
Weighted & 0.079 (0.085) & 0.080 (0.068) & \textbf{0.388 (0.728)} \\
Doubly robust   & 0.008 (0.007) & 0.008 (0.007) &  0.006 (0.005) \\
\bottomrule
\end{tabular}}
\end{small}
\end{center}
\vskip -0.1in
\label{table:synthetic}
\end{table}

Table \ref{table:synthetic} shows the doubly robust estimator is robust to either the regression model or the density ratio estimator mis-specification. As expected, the weighted regression can still fit an accurate linear relationship to the training points that are likely to occur under the target data distribution; but the unweighted regression suffers from the model mis-specification. Interestingly, even when the regression model is correct, if the density ratio estimates are biased, the weighted regression, despite being the consistent estimator, incurs large errors due to the limited training samples. The doubly robust estimator on the other hand is unaffected by the bias in density ratios as long as the regression model is correct. In Appendix \ref{table:synthetic_weighted2}, we additionally show that the effects of mis-specified density ratio estimates on $\hat V_{\text{w-soft}}$ are in/decreased by the training data size.

\subsection{Experiments on Clinical Domains} We now investigate the performance of our proposed estimators--namely, soft surrogate estimators with weighted and unweighted regression: $\hat V^{\pi_e}_{\text{soft}}, \hat V^{\pi_e}_{\text{w-soft}}$, and \textit{doubly robust} soft surrogate estimators also with weighted and unweighted regression: $\hat V^{\pi_e}_{\text{dr-soft}}, \hat V^{\pi_e}_{\text{dr-w-soft}}$-- in clinical simulators of HIV treatment and sepsis management. 

\paragraph{HIV treatment}
Human Immunodeficiency Virus (HIV) is a retrovirus that can lead to the lethal Acquired Immune Deficiency Syndrome (AIDS) devastating a person's immune system. In the simulator designed by~\cite{ernst}, the patient's state is a 6-dimensional vector representing the number of healthy and HIV-infected cells. Actions depend on Reverse Transcriptase Inhibitors (RTI) and Protease Inhibitors (PI). The behavioral policy chooses to administer either small or high dosage of RTI, and or PI (action size = 4) while the new target policy replaces the small dosage with zero. The motivation for designing such a target policy is to evaluate whether minimizing the drug prescription can help improve the patient's outcomes by avoiding unnecessary side-effects \cite{ernst2005}. The reward is given at every time step based on the number of free virus particles and the count of the HIV-specific T cells. The target long-term value is the sum of rewards across 200 time steps. We generate 2500 off-policy trajectories and only 500 on-policy samples under the new policy.

\paragraph{Sepsis management in ICU}
Sepsis is a life-threatening organ dysfunction and one of the leading causes of mortality in the United States~\citep{liu2014hospital}. We use the simulator by ~\cite{oberst2019counterfactual}, following the settings in~\cite{namkoong2020off}. The patient's state is represented by a binary indicator for diabetes, and four vital signs-- heart rate, blood pressure, oxygen concentration, glucose level--that take values in a subset of \{very\_high, high, normal, low, very\_low\}, leading to the state space of size 1440. The behavioral policy chooses from the two binary treatments: antibiotics, and mechanical ventilation, modeled after ~\cite{gao2024trajectory}. The new target policy has an additional option of vasopressors based on the implementation by ~\citet{oberst2019counterfactual}, thus operating in the larger action space of size 8. A reward of 0 or $\pm 1$ is given based on whether the patient is neutral, discharged, or dead. The full horizon length is 20 with a discount factor of 0.99. We generate 5000 off-policy trajectories and 500 on-policy samples. Details about the environments and the data-generation processes are in Appendix \ref{appendix:clinical simulators}.

\subsection{Baselines}
We compare our method to the following baselines: (i) LOPE estimator~\cite{saito2024long}, which also tries to relax the surrogacy assumption by additionally accounting for the action effects, (ii) Online-model based prediction, which requires learning a transition dynamics model from the limited on-policy data and extrapolating (unobserved) future states, (iii) extrapolation of the average short-term reward, (iv) extrapolation of the last short-term reward, and (v) full-horizon Monte Carlo Sampling. MC sampling in fact requires observing the full horizon outcome of the target policy, so it would not be feasible in settings of our interest where the long-term outcome takes a prohibitively long amount of time to observe. We include this to benchmark the performance of using a finite set of on-policy trajectories to estimate the long-term policy value. Even though MC estimates are unbiased, their variance may be high if the on-policy data is limited, and using the short-horizon soft surrogates may give smaller prediction errors due to the variance-bias trade-off. For space, we include a discussion of the baseline methods in Appendix \ref{appendix:baseline}.

 \begin{table*}[t!]
\centering
\caption{MSE of different estimators in HIV and Sepsis. The first 4 rows with (*) show our proposed estimators. Mean and std (in parentheses) for Sepsis are from 5 bootstrapping seeds. MC sampling is only evaluated in Sepsis due to the environment's stochasticity. The lowest MSE for a given $h$ is bold.}
\small
\resizebox{0.8\linewidth}{!}{\begin{tabular}{@{}lccc|cc@{}}
\toprule 
 & \multicolumn{3}{c}{\textbf{HIV (H=200)}} & \multicolumn{2}{c}{\textbf{Sepsis (H=20)} }\\  \cmidrule(l){2-6}
\textbf{Estimator} & \textbf{h=10}  & \textbf{20} & \textbf{50} & \textbf{h=2}  & \textbf{4} \\ \midrule
Weighted soft surrogate estimator (*)  & 71.51 & 53.21  & 53.85   & \textbf{0.03 (0.02)}  &  \textbf{0.01 (0.01)} \\
Soft surrogate estimator (*)  & \textbf{68.15} & 67.15  & 57.96   & 0.04 (0.01)  & 0.02 (0.01) \\
DR soft surrogate (*) & 68.32 & 64.86  & 58.15  & 0.04 (0.01)  & \textbf{0.01 (0.005)}  \\ 
DR weighted soft surrogate (*) & 70.90 & \textbf{50.61} &55.26 & 0.03 (0.01)  & 0.02 (0.004)  \\ \midrule
LOPE ~\cite{saito2024long}  & 69.95 & 67.75 & 71.91  & 0.14 (0.25)   & 0.10 (0.01)  \\
Online-model based & 404.58 & 402.49 & 137.37 & 0.62 (0.01) & 0.62 (0.01) \\
Average reward extrapolation  & 404.68 & 403.83 & 334.36 & 0.07 (0.01)  & 0.07 (0.01) \\ 
Last reward extrapolation  & 404.37 & 400.56 & \textbf{31.79} &  0.07 (0.01)  & 0.07 (0.01) \\
MC sampling & - & - & - & 0.02 (0.01)  & 0.02 (0.01) \\ 
\bottomrule
\end{tabular}}
\label{table:summary}
\end{table*}

\subsection{Results}

We first specify how the ground truth long-term policy value is defined. The HIV simulator assumes deterministic decision processes, so $V^{\pi_e}$ is computed as the average full-horizon returns over the 500 initial patient states. In Sepsis management, the transitions are stochastic, so we define the ground truth $V^{\pi_e}$ as the average MC estimates across 5000 roll-outs, which is 10 times larger than the on-policy sample size. In both simulations, the (ground-truth) target policy's value is higher than the historical policy's. In HIV, $V^{\pi_b}=337.01$ while $V^{\pi_e}=404.87$, and in Sepsis, $V^{\pi_b}=-0.178$ compared to $V^{\pi_e}=-0.006$.

Table \ref{table:summary} shows the prediction mean squared errors made by different estimators. Since the full horizon in HIV is 200, we consider making predictions at $h = 10 (5\% \text{of the full horizon}), 20 (10\%)$, and $50 (25\%)$, and in Sepsis which spans 20 time steps, we consider $h = 2 (10\%)$, and $4 (20\%)$. Even though in these domains, the Assumption 4.2, required for our theory, likely does not hold, we still observe that even at one-tenth of the full horizon, our estimators can yield accurate estimates of the long-term policy value. In practice, this could mean a reduction of 20 week-long clinical trials to just 2 weeks before determining the efficacy of new treatments. 

In contrast, the baselines either yield poor estimates or require the short-horizon length to be fairly long (e.g., $h=50$ in HIV) to make accurate predictions about the long-term outcome. In HIV, after $h = 50$, most patients have entered stable conditions, and therefore, extrapolating the last short-term reward can predict the long-term value most accurately. However, in Sepsis, where the patients receive no intermediate reward until they are either recovered or dead, using the last reward to extrapolate performs poorly compared to the regression estimates. Online-model based method also performs poorly in both domains because the on-policy data from the short horizon is too limited to learn an accurate transition model. Furthermore, the dynamics model's bias propagates through the prediction of future states, which may derail the overall prediction of future rewards if the prediction horizon is long. Further discussion about the experiment results is in Appendix \ref{appendix:additional_results}.

To show that our method generalizes to different target policies with varying policy values, we report the performance of our estimators on 7 other policies in Table \ref{supp_table:hiv} and \ref{supp_table:sepsis}. We observe that while the accuracy of the baseline methods fluctuate across target policies and domains, our method is able to provide consistently good estimates for different policies in both HIV and Sepsis. 

Rather than predicting the target policy's long-term value directly, decision-makers may be interested in quickly identifying if the new policy is likely to significantly out or under-perform the behavioral policy. In such cases, we can use the estimated policy values of the target policy to perform statistical testing comparing the new policy's mean to the behavioral policy's mean value. Specifically, we have 500 estimates for the on-policy samples, and 2500 or 5000 observed long-term outcomes in the historical data, which we can compare in the independent samples t-Test. We set the null hypothesis to be: ``The difference between the new policy's value and the behavioral policy's value is \textit{not} statistically significant," and evaluate how much on-policy data is necessary for successfully rejecting the null. In Sepsis, using only the 10\% of the full-horizon data, our estimators can reject the null with p-values less than $< 10^{-6}$, suggesting that the difference between the new policy and the behavioral policy's long-term outcomes is statistically meaningful. In HIV, $\hat V_{\text{soft}}$ and $\hat V_{\text{dr-w-soft}}$ give p-values less than $< 10^{-6}$ when $h = 10$ and 20. When $h = 50$, all our estimators successfully reject the null. These results suggest that our approaches may be especially useful for quickly identifying if the new policy is meaningfully different from the existing policy.

\section{Discussion \& Conclusion}
% While our primary motivation for long-term policy evaluation with novel actions is grounded in healthcare, 
Other than healthcare, there are also scenarios in education and online commerce, where a new policy introduces novel actions but requires a substantial amount of time to evaluate. For example, recent work has proposed new hybrid human-AI approach in which all students are additionally supported by human tutors when completing their online assignments ~\citep{thomas2024improving,abdelshiheed2024aligning}. Existing datasets only with human tutors or AI tutors will fail to have coverage for the hybrid actions prescribed by the new policy. At the same time, testing this new approach can be burdensome since it requires real-life student interaction over long horizons. Our method offers an approach of estimating the long-term outcome (e.g., end-of-year assessment scores) based on ``soft" short-term surrogates (e.g., student's daily activity logs).

We relax the surrogate index assumption that is difficult to satisfy in most real-world scenarios by adopting the perspective of ``dynamic invariance" ~\cite{battochi2021} to sequential RL policies. We propose estimators based on soft surrogates and give a finite sample analysis of doubly robust estimators in the setting of covariate shift. While the assumptions required for theory may still be strong, our empirical results show that in two key clinical simulators of HIV and sepsis management, soft surrogate estimators can reduce the observation window to be 10\% of the original long horizon.

Our experiments evaluate the estimators across different values of $h$. How to select $h$, given the potential trade-off between more signal about the new policy versus observation costs and increased covariate shift between the training and the test data, is an interesting question for future work.  

% Our experiments include the performance of the estimators across different values of $h$, but the selection of the short-horizon value remains an open question. Especially in scenarios, where decision makers can decide how long to run experiments for, the selection of $h$ for the desired level of prediction accuracy is important to have in practice. There may be practical considerations like trade-offs between more signal from the on-policy data versus increased observation costs; and there may also be a tension due to a larger model class (with the increased covariate dimension from a longer horizon $h$) and a relatively small sample size. Other interesting direction for future work includes designing a behavioral policy such that the resulting data can help build estimators that are accurate for some unknown target policy or a class of target policies.

\clearpage

\onecolumn

\section{Theory}\label{appendix:theory}

\subsection{Implication of Assumption \ref{assmp:surrogacy}}
We first define $V^{\pi_e \pi_b}$ as follows:
\begin{eqnarray*}
V^{\pi_e\pi_b} \equiv \mathbb E_{(s_0,r_0,s_1,r_1,\ldots,s_h,r_h) \sim \pi_e}  \left(\sum_{j=1}^h r_j + \right. 
\left. \mathbb E_{(s_{h+1},r_{h+1},\ldots,s_H,r_H) \sim \pi_b | (s_0,r_0,\ldots,s_h,r_h)}  \sum_{j'=h+1}^{H} r_{j'} \right),
\end{eqnarray*} which has two terms: the first term is determined by the short-term data, and the second term represents the (unobserved) long-term data from steps $h+1, ..., H$. 

Then we apply Assumption \ref{assmp:surrogacy} to the second term.

\begin{eqnarray*}
= \mathbb E_{(s_0,\ldots,r_h) \sim \pi_e}  \left(\sum_{j=1}^h r_j +  \mathbb E_{(s_{h+1},\ldots,r_H) \sim \pi_e | (s_0,\ldots,r_h)}  \sum_{j'=h+1}^{H} r_{j'} \right) = V^{\pi_e},
\end{eqnarray*} which is target estimand. This suggests that without the assumption, our estimators are predicting $\hat V^{\pi_e\pi_b}$, and our analysis hereafter applies to $V^{\pi_e \pi_b}$. The soft surrogacy assumption allows establishing the necessary connection to the target estimand, i.e., $V^{\pi_e \pi_b} = V^{\pi_e}$

\subsection{Baseline Regression Estimator}
\label{sec:baseline_regr_estimator}
The second term of the DR estimator represents the estimated value of 
$V^{\pi_e}$ using the regression estimate, and leverages the on policy data samples from the target policy $\pi_e$. We note that in our setting there are several choices for this quantity, and below we briefly motivate our choice. (For notation simplicity, we define $N = |\mathcal D_b|$ and $M= |\mathcal D_e|$. To avoid confusion with subscript $h$ denoting the short-horizon length, we use $a$ for density ratio, i.e., $a(\tau_h) \equiv \frac{p(\tau_h|\pi_e)}{p(\tau_h|\pi_b)}$.) Consider two possible options: 
\begin{eqnarray}
\frac{1}{N} \sum_{j=1}^{N} f(\tau^j_h)h(\tau^j_h) &\approx& E_{\tau_h \sim \pi_b} [ f(\tau_h)h(\tau_h)  ]\label{eqn:h_regr}\\
\frac{1}{M} \sum_{i=1}^{M} f(\tau^i_h) &\approx & E_{\tau_h \sim \pi_e} [ f(\tau_h)  ] \label{eqn:sh_regr}
\end{eqnarray}
Note the top expression uses the historical data and the bottom expression uses the on-policy short horizon data. We first consider when the regressor $f$ is accurate. Then the error in the estimate for Equation~\ref{eqn:sh_regr} is only due to the finite sample approximation of the expectation, and will generally decrease as $O(\frac{1}{M})$. 

We now consider the estimate in Equation~\ref{eqn:h_regr} which uses the density ratio (short-horizon trajectory propensity weights) $a(\tau_h)$. Let $a^*$ be the true (unknown) density ratio and $a$ be the density ratio estimated using the historical data and on policy short horizon data:
\begin{eqnarray}
\frac{1}{N} \sum_{j=1}^{N} f(\tau_h^j)a(\tau_h^j) 
&=& \frac{1}{N} \sum_{j=1}^{N} f(\tau_h^j)a(\tau_h^j) 
- f(\tau_h^j)a^*(\tau_h^j) + f(\tau_h^j)a^*(\tau_h^j) \\
&=& \frac{1}{N} \sum_{j=1}^{N} f(\tau_h^j)\left(a(\tau_h^j) -a^*(\tau_h^j)\right) + f(\tau_h^j)a^*(\tau_h^j) \\
&=&  \mathbb E_{\tau_h \sim \pi_b} [ f(\tau_h)a^*(\tau_h)  ] +  
\left(\frac{1}{N} \sum_{j=1}^{N} f(\tau_h^j)a^*(\tau_h^j) -  \mathbb E_{\tau_h \sim \pi_b} [ f(\tau_h)a^*(\tau_h)  ] \right) \nonumber \\
&&+ \frac{1}{N} \sum_{j=1}^{N} \left(f(\tau_h^j)(a(\tau_h^j) -a^*(\tau_h^j)\right),
\end{eqnarray}
The second term is (like above) due to using finite samples to approximate an expectation and will therefore generally scale with the size of the historical dataset, $O(\frac{1}{N}).$  The third term involve the error in the density ratio. Estimating the density ratio $a$ uses short-horizon on policy data, and so in general we expect the estimation error in $a$ to at best scale with  $O(\frac{1}{M})$, which implies we expect the overall error 
\begin{eqnarray}
\frac{1}{N} \sum_{j=1}^{N} f(\tau_h^j)a(\tau_h^j) 
-  \mathbb E_{\tau_h \sim \pi_b} \left[ f(\tau_h)a^*(\tau_h)  \right] = O \left( \frac{1}{N} + \frac{1}{M} \right). 
\end{eqnarray}
Therefore we expect the finite-sample error to be better by using the on-policy data estimator Equation~\ref{eqn:sh_regr}. Note that in standard offline reinforcement learning, no online data from the target policy is available and this "baseline" term is quite different: it is typically a weighting over the target policy, aka $V^{\pi_e}(s) = \sum_{a \in \mathcal A} p(a|s; \pi_e) Q^{\pi_e}(s,a)$ where $Q^{\pi_e}$ is estimated on the offline data from $\pi_b$, and $p(a|s; \pi_e)$ is known because that is simply the new target policy. 

\subsection{DR Estimator Consistency}\label{appendix:dr_consistency}
We first recall some prior results. First note that in our setting we are interested in solving a moment of the following form
\begin{flalign}
    \theta_0 =~& \mathbb E_{\pi_e}[m(Z;f_0)] &
    f_0(X) =~& \mathbb E_{\pi_b}[Y\mid X]
\end{flalign}
where $m(Z;f_0)$ is a linear functional of $f_0$ and (with a slight abuse of notations), $\pi_e, \pi_b$ denote two different distributions of the data under the new target policy (i.e., short-term data) and the behavioral policy (i.e., long-term data). Here the goal is a linear functional of a regression function $f_0$ over the distribution of the on-policy short-term data, but we trained $f_0$ as a regression over off-policy long-term data. Let $a(x) = \frac{p(X; \pi_e)}{p(X; \pi_b)}$ denote the density ratio of the two distributions. For consistency of notations with prior work~\cite{chernozhukov2017doubledebiasedmachinelearningtreatment}, we replace $\tau_h$ (input to the regression model) with $X$ and $G$ (outcome variable of interest) with $Y$.

Our example is a special case of this where:
\begin{align}
    \theta_0 =~& \mathbb E_{\pi_e}[f_0(X)] & f_0(X) = \mathbb E_{\pi_b} [Y\mid X]
\end{align}
where $X$ is the vector of surrogates. 

Any such linear functional has a doubly robust representation:
\begin{align}
    \theta_0 = \mathbb E_{\pi_e}[m(Z;f_0)] + \mathbb E_{\pi_b}[a_0(X)\,(Y - f_0(X))]
\end{align}
where $a_0(X)$ is the Riesz representer (the element in the Hilbert space, that represents the linear functional as an inner product; with respect to the $L^2$ inner product over the distribution in $\ell$). In other words, it is the function that has the property that:
\begin{align}
    \forall g: \mathbb E_{\pi_e}[g(X)] = \mathbb E_{\pi_b}[a_0(X)\, g(X)]
\end{align}

In our case, this Riesz representer is the density ratio:
\begin{align}
    a_0(X) = \frac{p_{\pi_e}(X)}{p_{\pi_b}(X)}
\end{align}

We now note that the error due to the nuisance parameters for the DR representation of such linear functionals can be bounded by the error in the two nuisance parameters (i.e., $f, a$): 
\begin{theorem}[Doubly Robust Bias Bound]\cite{chernozhukov2023automatic}\label{appendix:dr_bias_bound}
    Consider the population moment function
    $$M(g,a) = \mathbb E_{\pi_e}[m(Z;g)] + \mathbb E_{\pi_b}[a(X)\, (Y - g(X))]$$
    Then we have that for all $a$:
    \begin{align}
        M(g, a) - M(f_0, a_0) &=& \mathbb E_{\pi_b}[(a(X) - a_0(X))\, (f_0(X) - g(X))]\\
        &\leq& \mathbb E_{\pi_b} [(a^{(k)} - a_0)^2 ] \mathbb E_{\pi_b} [(f^{(k)} - f_0)^2 ]
    \end{align}
    \label{thm:bound_nuisance_impact}
\end{theorem}
\begin{proof}
    
    \begin{align*}
        M(g, a) - M(f_0, a_0) =~& \mathbb E_{\pi_e}[m(Z;g)] + \mathbb E_{\pi_b}[a(X)\, (Y - g(X))] - \mathbb E_{\pi_e}[m(Z;f_0)] -\mathbb E_{\pi_b}[a_0(X)\, (Y - f_0(X))]\\
        =~& \mathbb E_{\pi_e}[m(Z;g)] + \mathbb E_{\pi_b}[a(X)\, (f_0(X) - g(X))] - \mathbb E_{\pi_e}[m(Z;f_0)] -\mathbb E_{\pi_b}[a_0(X)\, (f_0(X) - f_0(X))]\\
        =~& \mathbb E_{\pi_e}[m(Z;g)] + \mathbb E_{\pi_b}[a(X)\, (f_0(X) - g(X))] - \mathbb E_{\pi_e}[m(Z;f_0)],
    \end{align*}
    where the first equality follows from the definition of the moment $M(g,a)$, and the second equality uses the tower rule of expectation and the fact that $f_0(X)=\mathbb E_{\pi_b}[Y\mid X]$.
    
    Note that $a_0$ satisfies that for all $g$:
    \begin{align}
        \mathbb E_{\pi_e}[m(Z; g)] = \mathbb E_{\pi_b}[a_0(X)\,g(X)]
    \end{align}
    we have that:
    \begin{align*}
        M(g, a) - M(f_0, a_0) =~& \mathbb E_{\pi_b}[a_0(X) g(X)] + \mathbb E_{\pi_b}[a_0(X)\, (f_0(X) - f_0(X))] - \mathbb E_{\pi_b}[a_0(X) f_0(X)]\\
        =~& \mathbb E_{\pi_b}\left[ a_0(X)\, (g(X) - f_0(X)) + a(X)\,(f_0(X) - g(X))\right]\\
        =~& \mathbb E_{\pi_b}\left[(a(X) - a_0(X))\, (f_0(X) - g(X))\right] \\
        \leq~& \mathbb E_{\pi_b} [(a^{(k)} - a_0)^2]  \mathbb E_{\pi_b} [(f^{(k)} - f_0)^2], 
    \end{align*}
where the final inequality follows from the Cauchy–Schwarz inequality.
\end{proof}
% https://en.wikipedia.org/wiki/Cauchy%E2%80%93Schwarz_inequality
% please do not use Grammarly to correct this

Recall Equations (9) $\sim$ (11) as defined in Section \ref{section:theory}
  \begin{eqnarray*}
\epsilon_e \equiv  \epsilon_{e} (f^{(k)}) &=& \sqrt{\mathbb E_{\tau_h \sim \pi_e}[(f^{(k)}(\tau_h) - f_0(\tau_h))^2]} \\
\epsilon_b \equiv  \epsilon_{b} (f^{(k)}) &=& \sqrt{\mathbb E_{\tau_h \sim \pi_b}[(f^{(k)}(\tau_h) - f_0(\tau_h))^2]} \\ \epsilon_h \equiv
 \epsilon_{b}(a^{(k)}) &=& \sqrt{\mathbb E_{\tau_h \sim \pi_b}[(a(\tau_h) - a_0(\tau_h))^2]}
    \end{eqnarray*}
\begin{theorem}[Variance-Based Rate for DR]\label{thm:variance-dr}
    % Let $\hat{\theta}$ be the cross-fitted doubly robust estimate:
    % \begin{align}
    %     \hat{\theta} = \frac{1}{K} \sum_{k=1}^K \frac{1}{|I_s \cap I_k|} \sum_{i\in I_s\cap I_k} f^{(k)}(Z_i) + \frac{1}{|I_\ell \cap I_k|}\sum_{i\in I_\ell \cap I_k} a^{(k)}(Z_i)\, (G^i_{H} - f^{(k)}(Z_i))
    % \end{align}
   Assume that $|G(\tau_h)|, |f^{(k)}(\tau_h)|, |f_0(\tau_h)|$ are a.s. bounded by $C_1H$ and $a^{(k)}(\tau_h), a_0(\tau_h)$ are a.s. bounded by $C_2$, where $C_1$ and $C_2$ are constants. Let $N = |\mathcal D_b|$ (i.e., historical off-policy dataset) and $M = |\mathcal D_e|$ (i.e., short-term on-policy dataset).
   Then w.p. at least $1-\delta$:
    \begin{eqnarray*}
        |\hat V^{\pi_e}_{DR}-V^{\pi_e}| &\leq& \sqrt{\frac{2\text{Var}_{\tau_h \sim \pi_e}(f_0(\tau_h))\log(4K/\delta)}{M}}  + \sqrt{\frac{2\mathbb E_{\tau_h \sim  \pi_b}\left[a_0(\tau_h)^2\,(G(\tau_h) - f_0(\tau_h))^2\right]\log(4K/\delta)}{N}}    \\
         &&+ \frac{1}{K}\sum_{k=1}^K  \epsilon_{b} (f^{(k)}) \cdot \epsilon_{b}(a^{(k)}) +  \frac{2 H K\log(4K/\delta)}{M}  
         + \max_{k}\epsilon_{e} (f^{(k)}) \sqrt{\frac{2K\log(4K/\delta)}{M}}  \\
&&          + \frac{4C_1 C_2 H K\log(4K/\delta)}{N} 
          + 3 C_1 H \!\max_{k}\left\{
    \epsilon_{b} (f^{(k)})  + \!\epsilon_{b}(a^{(k)})\right\}\!\!\sqrt{\frac{2K\log(4K/\delta)}{N}} 
    \end{eqnarray*}
    \label{thm:var_bound}
    
\end{theorem}
\begin{proof}
We can write:
\begin{align}
    \hat V^{\pi_e}_{DR}-V^{\pi_e} 
    =~& \frac{1}{K} \sum_{k=1}^K 
    \left(
    \frac{1}{|I_{e} \cap I_k|} \sum_{i\in I_{e}\cap I_k} (f^{(k)}(Z_i) - V^{\pi_e}) + \frac{1}{|I_{_b} \cap I_k|}\sum_{i\in I_b \cap I_k} 
    \underbrace{a^{(k)}(Z_i) (G_i - f^{(k)}(Z_i))}_{
    \triangleq B_i(f^{(k)},a^{(k)})
    }
    \right)
    \\ 
    =~&  \frac{1}{K} \sum_{k=1}^K \left(\frac{1}{|I_e \cap I_k|} \sum_{i\in I_e\cap I_k} (f_0(Z_i) - V^{\pi_e}) + \frac{1}{|I_b \cap I_k|}\sum_{i\in I_b \cap I_k} B_i(f_0, a_0)\right)\\
    ~&~~ + \frac{1}{K} \sum_{k=1}^K \left(\frac{1}{|I_e \cap I_k|} \sum_{i\in I_e\cap I_k} (f^{(k)}(Z_i) - f_0(Z_i)) + \frac{1}{|I_b \cap I_k|}\sum_{i\in I_b \cap I_k} (B_i(f^{(k)}, a^{(k)}) - B_i(f_0, a_0))\right)\\
    =~&  \frac{1}{M} \sum_{i\in I_e}  (f_0(Z_i) - V^{\pi_e}) + \frac{1}{N}\sum_{i\in I_b} B_i(f_0, a_0)\\
    ~&~~ + \frac{1}{K} \sum_{k=1}^K \left(\frac{1}{|I_e \cap I_k|} \sum_{i\in I_e\cap I_k} (f^{(k)}(Z_i) - f_0(Z_i)) + \frac{1}{|I_b \cap I_k|}\sum_{i\in I_b \cap I_k} (B_i(f^{(k)}, a^{(k)}) - B_i(f_0, a_0))\right)
\end{align}

The first term is a sum of $M$ i.i.d. mean-zero random variables, since $V^{\pi_e}=\mathbb E_{X \sim \pi_e}[f_0(X)]$. Hence, by a Bernstein bound, we have that, w.p. $1-\delta$:
\begin{flalign}
    \left|\frac{1}{M}\sum_{i\in I_e} (f_0(Z_i)-V^{\pi_e})\right| \leq \sqrt{\frac{2\text{Var}_{X \sim \pi_e}(f_0)\log(1/\delta)}{M}} + \frac{2H\log(1/\delta)}{M} \triangleq {\cal E}_1(\delta)
\end{flalign}

Similarly, the second term is the sum of $N$ i.i.d. mean-zero random variables, since $\mathbb E_{Z \sim \pi_b}[G \mid Z] = f_0(Z)$. Hence, by a Bernstein bound, we have that, w.p. $1-\delta$:
\begin{flalign}
    \left|\frac{1}{N}\sum_{i\in I_b} B_i(f_0, a_0)\right| \leq~& \sqrt{\frac{2\text{Var}_{\pi_b}(B(f_0, a_0))\log(1/\delta)}{N}} + \frac{2H^2\log(1/\delta)}{N}\\
    =~& \sqrt{\frac{2\mathbb E_{\pi_b}[a_0(Z)^2 (G - f_0(Z))^2]\log(1/\delta)}{N}} + \frac{2H^2\log(1/\delta)}{N}\triangleq {\cal E}_2(\delta)
\end{flalign}

For each fold $k$, let:
\begin{align}
    \hat{\Delta}_k(f) =~& \frac{1}{|I_e\cap I_k|} \sum_{i\in I_e\cap I_k} (f(Z_i) - f_0(Z_i)) & \Delta(g) =~& \mathbb{E}[\hat{\Delta}_k(g)] = \mathbb E_{\pi_e}[f(X) - f_0(X)]\\
    \hat{\Lambda}_k(f, a) =~& \frac{1}{|I_b\cap I_k|} \sum_{i\in I_b\cap I_k} (B_i(f, a) - B_i(f_0, a_0)) & \Lambda(g) =~& \mathbb E[\hat{\Lambda}_k(f)] = \mathbb E_{\pi_b}[a(Z)\,(G(Z) - f(Z))]
\end{align}
Conditional on the folds $(I_1, \ldots, I_K)$, and on $I_e, I_b$ and on the estimates $f^{(k)}, \hat{\alpha}^{(k)}$, we have that $\hat{\Delta}_k(f^{(k)})$, is an average of $|I_k\cap I_b|$ i.i.d. random variables with mean $\Delta(f^{(k)}) = \mathbb E_{\pi_e}[f^{(k)}(Z) - f_0(Z)]$. Since, we assumed that $|f^{(k)}(Z)|, |f_0(Z)| \leq H C_1$, a.s., by a Bernstein bound we have that w.p. $1-\delta$:
\begin{align}
    |\hat{\Delta}_k(f^{(k)}) - \Delta(f^{(k)})| \leq~& \sqrt{\frac{2\mathbb E_{\pi_e}\left[(f^{(k)}(Z)-f_0(Z))^2\right]\log(1/\delta)}{|I_e\cap I_k|}} + \frac{2H C_1\log(1/\delta)}{|I_e\cap I_k|}\\
    =~& \epsilon_{e}(f^{(k)}) \sqrt{\frac{2K\log(1/\delta)}{M}} + \frac{2H C_1 K\log(1/\delta)}{M}
\end{align}
Similarly, $\hat{\Lambda}_k(f^{(k)}, a^{(k)})$, is an average of $|I_k\cap I_b|$ i.i.d. random variables with mean $\Lambda(f^{(k)}, a^{(k)}) = \mathbb E_{\pi_e}[a^{(k)}(Z)\, (G(Z) - f^{(k)}(Z)]$. 
    
Also recall we assumed that $|G|, |f^{(k)}|, |f_0|$ are a.s. bounded by $C_{1}H$ and $\hat{a}^{(k)}, a_0$ are a.s. bounded by $C_2$, by a Bernstein bound we have that w.p. $1-\delta$:
\begin{align}
    |\hat{\Lambda}_k(f^{(k)}, a^{(k)}) - \Lambda(f^{(k)}, a^{(k)})| \leq \sqrt{\frac{2\mathbb E_{\pi_b}\left[(B_i(f^{(k)}, a^{(k)})- B_i(f_0, a_0))^2\right]\log(1/\delta)}{|I_k\cap I_b|}} + \frac{4H C_1 C_2 \log(1/\delta)}{|I_k\cap I_b|}
\end{align}
Note that:
\begin{align*}
    B_i(g, a)- B_i(f_0, a_0) =~& a(Z)\,(G - f(X)) - a_0(Z)\,(G - f_0(X))\\
    =~& a(Z)\,(G - g(X)) - a_0(Z)\,(G - f(X)) + a_0(Z)\,(f_0(X) - f(X))\\
    =~& (a(Z) - a_0(Z)) (G - f(X)) + a_0(Z) (f_0(X) - f(X))
\end{align*}
Note $(a + b)^2 \leq 2 (a^2 + b^2)$, and $\sqrt{a + b} \leq \sqrt{a} + \sqrt{b}$ (for any $a,b\geq 0$), and again using our assumption  that $|G|, |f^{(k)}|, |f_0|$, $\hat{a}^{(k)}, a_0$ are  almost surely bounded,  we have that:
\begin{align*}
    \sqrt{\mathbb E_{\pi_b}[(B_i(g, a)- B_i(f_0, a_0))^2]} \leq~& \sqrt{2\mathbb E_{\pi_b}\left[4 C_1^2 H^2 (a(Z)-a_0(Z))^2 + C_2^2  (f(X)-f_0(X))^2\right]}\\
    \leq~& \left(3 C_1 H \epsilon_{b}(a^{(k)}) + 2C_2 \epsilon_{b}(f^{(k)})\right)\\
    \leq~& 3 \left(H C_1\,\epsilon_{b}(a^{(k)})  + C_2\epsilon_{b}(f^{(k)})\right)
\end{align*}
Thus we conclude that:
\begin{align*}
    |\hat{\Lambda}_k(f^{(k)}, a^{(k)}) - \Lambda(f^{(k)}, a^{(k)})| \leq~& \left(H C_1\,\epsilon_{b}(a^{(k)})  + C_2\epsilon_{b}(f^{(k)})\right). \sqrt{\frac{2\log(1/\delta)}{|I_k\cap I_b|}} + \frac{4HC_1 C_2\log(1/\delta)}{|I_k\cap I_b|}\\
    \leq~& 3 ( H C_1\,\epsilon_{b}(a^{(k)})  + C_2\epsilon_{b}(f^{(k)}))
    \sqrt{\frac{2K\log(1/\delta)}{N}} + \frac{4HC_1 C_2 K\log(1/\delta)}{N}
\end{align*}

By applying a union bound over all the aforementioned $2K+2\leq 4 K$ bad events, we get that with probability $1-\delta$
\begin{align*}
    |\hat{V}^{\pi_e}_{DR} - V^{\pi_e}| \leq~&  {\cal E}_{1}(\delta / 4K) + {\cal E}_{e}(\delta / 4K) + \frac{1}{K} \sum_{k=1}^K (\Delta(f^{(k)}) + \Lambda(f^{(k)}, a^{(k)})) \\
    ~&~~ + \max_{k=1}^K \epsilon_{e}(f^{(k)}) \sqrt{\frac{2K\log(4K/\delta)}{M}} + \frac{2H K\log(4K/\delta)}{M}\\
    ~&~~ + 3 \max_{k=1}^K 
    3 \{ H C_1\,\epsilon_{b}(a^{(k)})  + C_2\epsilon_{b}(f^{(k)})\}
    \sqrt{\frac{2K\log(1/\delta)}{N}} + \frac{4H C_1 C_2 K\log(4K/\delta)}{N}
\end{align*}
Finally, note that:
\begin{align}
    \Delta(f) + \Lambda(f, a) = M(f, a) - M(f_0, a_0)
\end{align}
with $M(f, a)$ as defined in Theorem~\ref{thm:bound_nuisance_impact}. Then  invoking Theorem~\ref{thm:bound_nuisance_impact}, we also get that:
\begin{align}
    \Delta(f) + \Lambda(f, a) \leq  \epsilon_{b}(a^{(k)})\epsilon_{b}(f^{(k)})
\end{align}
Combining all the above yields the theorem.

\end{proof}

% \begin{align*}
%     \mathbb E[\hat V_{DR}^{\pi_e}]=~& \mathbb E\left[ \frac{1}{K} \sum_{k=1}^K \left(\frac{1}{|I_s \cap I_k|} \sum_{i\in I_s\cap I_k} (f^{(k)}(Z_i) + \frac{1}{|I_\ell \cap I_k|}\sum_{i\in I_\ell \cap I_k} a^{(k)}(Z_i)\, (G^i_{H} - f^{(k)}(Z_i))\right) \right] -  \\
%     =~&  \frac{1}{K} \sum_{k=1}^K E\left[\left(\frac{1}{|I_s \cap I_k|} \sum_{i\in I_s\cap I_k} (f^{(k)}(Z_i) + \frac{1}{|I_\ell \cap I_k|}\sum_{i\in I_\ell \cap I_k} a^{(k)}(Z_i)\, (G^i_{H} - f^{(k)}(Z_i))\right) \right] 
% \end{align*}

\begin{theorem}
Given Assumption \ref{assmp:surrogacy}, and under the the same assumptions as Theorem~\ref{thm:var_bound}, if our estimates of the regression function $f$ and propensity weights / density ratio $a$ are asymptotically consistent, at any rate, then 
$\hat V^{\pi_e}_{DR}$ is a consistent estimator of $V^{\pi_e}$.
\label{eqn:dr_consistent}
\end{theorem}

\begin{proof}
The result immediately follows from Theorem~\ref{thm:var_bound}. Given $|G_H|, |f^{(k)}(\tau_h)|, |f_0(\tau_h)|$, $a^{(k)}(\tau_h)$, and $a_0(\tau_h)$ are a.s. bounded, terms 1,2,4 and 6 all go to zero as as $N$ and $M$ go to infinity. Terms 3, 5 and 7 also all go to zero when the nuisance parameters $f$ and $a$ are asymptotically consistent.
\end{proof}

\begin{theorem}
\label{theorem-productbias}
Define the error in the predicted value as $\Delta(\tau_h,f) = f(\tau_h) - f_0(\tau_h)$ 
and the  density ratio relative to the true density ratio as $\delta(\tau_h) =  \frac{a(\tau_h)}{a_0(\tau_h)} $ (where $f_0$ and $a_0$ are the true regression function and the true density ratio). 
Under Assumptions \ref{assmp:surrogacy}, \ref{assmp:coverage}, the bias of $\hat{V}_{DR}^{\pi_e}$ is 
\begin{equation}
\hat V_{DR}^{\pi_e} - V^{\pi_e} = \frac{1}{K} \sum_{k=1}^K \mathbb E_{\tau_h \sim \pi_e} \left[\Delta(\tau_h,f^{(k)}) \left(1-\delta^{(k)}(\tau_h)\right) \right],
\end{equation}
the average of the product-bias terms across the folds.
\end{theorem}
\begin{proof}

Without loss of generalizability, we will now analyze this for a single fold $k$. 
Since each trajectory within a fold, is independent and identically distributed, we will analyze the expectation for a single term with a single fold. We also now make the expectation term explicit. We denote $\hat V^{\pi_e}_{k}$ as the value of the DR-estimator for the $k$-th fold.
\begin{align}
\mathbb E_{\tau_h \sim \pi_b; \tau'_h \sim \pi_e} [\hat V^{\pi_e}_{k}] =~& 
\mathbb E_{\tau'_h \sim \pi_e}[f^{(k)}(\tau'_h)] + \mathbb E_{\tau_h \sim \pi_b}[ a^{(k)}(\tau_h)\, (G - f^{(k)}(\tau_h)) ] \nonumber \\
=~& 
\mathbb E_{\tau'_h \sim \pi_e}[f_0(\tau'_h) + \Delta(\tau_h',f^{(k)})] + \mathbb E_{\tau_h \sim \pi_b}[ a^{(k)}(\tau_h)\, (G - (f_0(\tau_h) + \Delta(\tau_h,f^{(k)})) ] \nonumber \\  
=~& 
\mathbb E_{\tau_h' \sim \pi_e}[f_0(\tau'_h) + \Delta(\tau_h',f^{(k)})] + \mathbb E_{\tau_h \sim \pi_b}[ a_0(\tau_h)\delta^{(k)}(\tau_h)\, (G - (f_0(\tau_h) + \Delta(\tau_h,f^{(k)})) ] \nonumber  \\  
=~& 
\mathbb E_{\tau'_h \sim \pi_e}[f_0(\tau'_h) + \Delta(\tau_h',f^{(k)})] + \mathbb E_{\tau_h \sim \pi_e}[ \delta^{(k)}(\tau_h)\, (G - (f_0(\tau_h) + \Delta(\tau_h,f^{(k)})) ] \nonumber  \\  
=~& 
\mathbb E_{\tau'_h \sim \pi_e}[f_0(\tau'_h) + \Delta(\tau_h',f^{(k)})] + \mathbb E_{\tau_h \sim \pi_e}[ \delta^{(k)}(\tau_h)\, (f_0(\tau_h) - (f_0(\tau_h) + \Delta(\tau_h,f^{(k)}))) ] \nonumber  \\  
=~& 
\mathbb E_{\tau'_h \sim \pi_e}[f_0(\tau'_h) + \Delta(\tau_h',f^{(k)})] + \mathbb E_{\tau_h \sim \pi_e}[  - \delta^{(k)}(\tau_h)\,\Delta(\tau_h,f^{(k)}) ] \\  
=~& 
\mathbb E[V_k^{\pi_e}] + 
\mathbb E_{\tau_h \sim \pi_e}[ \Delta(\tau_h,f^{(k)}) (1  - \delta^{(k)}(\tau_h)) ] \\  
\end{align}
where the first equality holds because the first term depends only on trajectories from $\pi_e$ and the second is a function of trajectories from $\pi_b$; the second and third equalities use the definition of $\delta$ and $\Delta$; the fourth equality uses the definition of $a_0$ to replace the expectation over $\pi_b$ to an expectation over $\pi_e$; the fifth equality, with slight abuse of notation, takes an expectation over the randomness in the observed reward for a given state trajectory and rely on $f_0$ being the true regression function, $\mathbb E [G(\tau_h)]= f_0(\tau_h)$; the sixth equality simplifies the prior expression; and the seventh equality uses that $\mathbb E_{\tau'_h \sim \pi_e}[f_0(\tau'_h)] = \mathbb E [V_k^{\pi_e}]$. The theorem follows by averaging over each of the folds.
\end{proof}

\begin{theorem}
If the regression function $f$ is asymptotically consistent, then 
$\hat V^{\pi_e}_{\text{soft}}$ is a consistent estimator of $V^{\pi_e}$.
\label{eqn:slev_consistent}
\end{theorem}
\begin{proof}
Recall 
\begin{eqnarray}
|\hat V^{\pi_e}_{\text{soft}} - V^{\pi_e}| &=& \left|\frac{1}{M} \sum_{i=1}^{M} f(\tau_{h}^j) - V^{\pi_e} \right| \\
&=&\left| \frac{1}{M} \sum_{i=1}^{M} f(\tau_{h}^i) - f_0(\tau_{h}^i) +f_0(\tau_{h}^i) - V^{\pi_e}  \right| \\
&\leq& \frac{1}{M} \sum_{i=1}^{M} \left|f(\tau_{h}^i) - f_0(\tau_{h}^i)\right| + \left| \frac{1}{M} \sum_{i=1}^{M}  f_0(\tau_{h}^i) - V^{\pi_e} \right|\\
&\to& 0
\label{eqn:slr}
\end{eqnarray}
In the third line, the first term goes to zero if $f$ is asymptotically consistent, and the second term is the error due to the finite sample approximation of the expectation $\mathbb E_{\tau_h \sim \pi_e}[.]$, and also goes to zero asymptotically. 
\end{proof}

\section{Experiments}

We design a toy environment to evaluate the robustness of the proposed estimators to the regression model and the density ratio estimator mis-specification. We assume a historical dataset of full-horizon observations as well as a smaller on-policy dataset of the new target policy but only executed for short horizons. First we describe the details about the environment's states, transitions, returns, and surrogacy.

\paragraph{Domain}Initial states $s_0$ are evenly spaced between [0, 1.5] and added by a Gaussian noise $\sim N(0, 0.1)$. Under the behavioral policy $\pi_b$, transitions to the next state are as follows:

\[
s_{t+1} =
\begin{cases} 
s_0 & \text{w.p 0.5} \\
(-0.6 + 0.1 * U[0,1)) * s_t & \text{w.p 0.45} \\
1.5 & \text{otherwise}
\end{cases}
\] Under the new target policy $\pi_e$, the next states are defined as:
\[
s_{t+1} =
\begin{cases} 
1.5 & \text{if $s_t < 1.25$} \\
0 & \text{otherwise}
\end{cases}.
\] Every state is added a noise from $N(0, 0.1)$. The true long-term return is defined by the quadratic function of $\tau_{0:1}$ as $f(s_0, s_1) = 5 s_0 + s_1 + s_1^2 + N(0, \omega)$ where $\omega$ determines the amount of noise, i.e., the higher the $\omega$, the more reliable the regression estimates are compared to Monte Carlo sampling. $|D_b| = 5000$, and $|D_e| = 100$. 

We fit an ordinary least squares regression model using `statsmodel` packages. A correct regression model is given by: $f_\theta(s_0, s_1) = \theta^\top [s_0, s_1, s_1^2]$. For density ratio estimation, continuous states are discretized into 50 bins (thus making a total of 2500 bins for $(s_0, s_1)$), and the density ratio estimator is built based on the number of occurrences of each bin under $\pi_e$ versus $\pi_b.$ We report the mean squared error of the true and the estimated returns of the 100 target trajectories under $\pi_e$: $\sqrt{\frac{1}{|D_{e}|}\sum_{i=1}^{100} (\hat V^{\pi_e}_i - V^{\pi_e}_i)^2}$. 
% The code for Table 1 is available as ``surrogate synthetic.py" and the remaining experiments with ``surrogate.py" using --Ntrain 5000 --Ntarget 100. 

\begin{table}[h]
\centering
\caption{We show the mean and standard deviation of Mean-Squared Error (MSE) over 200 random seed runs.}
% \begin{tabular}{@{}lcc@{}}
% \toprule
%  & \begin{tabular}[c]{@{}l@{}}Large noise \\ ($\omega = 10$)\end{tabular} & \begin{tabular}[c]{@{}c@{}}Small noise \\ ($\omega = 1$)\end{tabular} \\ \midrule
% Short long estimator with unweighted regressor   & \textbf{0.11 (0.15)}    &   \textbf{0.002 (0.002)}  \\
% Monte Carlo full returns estimate   &  103.63 (11.24)   &  0.99 (0.16)   \\ \bottomrule
% \end{tabular}
\begin{tabular}{@{}lcc@{}}
\toprule
& \begin{tabular}[c]{@{}l@{}}\textbf{Large noise} \\ ($\omega = 10$)\end{tabular} & \begin{tabular}[c]{@{}c@{}}\textbf{Small noise} \\ ($\omega = 1$)\end{tabular} \\ \midrule
Soft surrogate estimator  & \textbf{0.251 (0.307)}    &   \textbf{0.002 (0.003)}  \\
Monte Carlo sampling  &  98.457 (13.355)   &  0.984 (0.133)   \\ \bottomrule
\end{tabular}
\label{table:synthetic_noise}
\end{table}

As discussed in the main text, Table \ref{table:synthetic_noise} shows that our estimates using the short trajectory data provide more accurate predictions of $V^{\pi_e}$ than Monte Carlo sampling due to the high noise in the observed outcomes. This suggests that even when decision makers have the budget for full-horizon observations, it may still be beneficial to consider soft surrogate estimators due to the variance-bias trade-offs.

\paragraph{Model Mis-specification}We now consider the robustness of the estimators under the model mis-specification. As described above, a correct regression model is a quadratic function of $(s_0, s_1)$: $f(s_0, s_1) = 5 s_0 + s_1 + s_1^2$. However, we purposely use a mis-specified linear model, $f_{\tilde \theta}(s_0, s_1) = {\tilde \theta}^\top [s_0, s_1]$, which cannot describe the historical data distribution. Note that this (mis-specified) linear model can still capture the data distribution under the target policy because $s_1$ under $\pi_e$ is always going to be either 0 or 1.5. As a result we expect the weighted regression and the doubly robust estimator to be unaffected by the regression model mis-specification as long the density ratio estimates are correct, but the unweighted regression will fail to predict for the target distribution.

In order to bias the density ratio estimates, we add a non-zero gaussian noise $\sim N(10, 10)$ to the denominator of $\frac{p(s_0, s_1|\pi_e)}{p(s_0, s_1|\pi_b)}$. We expect the unweighted estimator and the doubly robust estimator to be unaffected by the density ratio mis-specification as the regressor model is still correct. While the weighted estimator is still asymptotically consistent, with finite samples, it may suffer from the bias in density ratio estimates.

\paragraph{Results}
\begin{table}[h]
\centering
\caption{We show the MSE of the estimators with noise $\omega = 1$, when either the regression model or the density ratio model is mis-specified. Mean and std are from 200 random seed runs.}

\resizebox{\linewidth}{!}{
\begin{tabular}{@{}lccc@{}}
\toprule 
& \textbf{Both Realizable} & \textbf{Regression Misspecified} & \textbf{Density Ratio Misspecified} \\ \midrule
Soft surrogate estimator & 0.002 (0.003)   & \textbf{0.914 (0.064)} & 0.002 (0.003) \\
Short surrogate weighted estimator   & 0.079 (0.085)   & 0.080 (0.068)   & \textbf{0.388 (0.728)}  \\ 
DR soft surrogate estimator & 0.008 (0.007)   &  0.008 (0.007) & 0.006 (0.005) \\ \bottomrule
\end{tabular}}
% }
\label{table:synthetic_weighted}
\end{table}

Table \ref{table:synthetic_weighted} compares the performance of the 3 different estimators under different types of mis-specification. As expected, the doubly robust method does well in all three cases. The weighted regression gives accurate predictions even when the regression model is mis-specified as long as the density ratio estimates are correct. On the other hand, when the density ratio estimates are incorrect, the unweighted regression and the DR estimator predict accurately. The weighted regression suffers from the bias in density ratio estimates; but since it is also a consistent estimator, the effects due to the incorrect density ratio estimates are lessened as the training data size increases and the bias in the density ratio estimates are swept away by the training data. We show that in Table \ref{table:synthetic_weighted2} as the training data size increases from 500 to 50000, the MSE drops from 2.123 to 0.259.

Our synthetic experiment demonstrates that the doubly robust estimator is robust to both kind of model mis-specification as long as the other model is correct. The unweighted regression model has high errors when the regression model is mis-specified; while the weighted regression shows high errors when the density ratio estimates are biased, which are exacerbated if the historical data size is small.

\begin{table}[h]
\centering
\caption{MSE of the weighted regression model across different training data sizes.}
\begin{tabular}{@{}lc@{}}
\toprule
Training data size &  Error \\ \midrule
500  & 2.123 (7.382) \\
1000  & 0.806 (1.281) \\
50000 & 0.259 (0.480) \\
\bottomrule
\end{tabular}
\label{table:synthetic_weighted2}
\end{table}

\section{Baseline Methods}
\label{appendix:baseline}

\paragraph{LOPE estimator} ~\citet{saito2024long} proposes an estimator that can be robust to the violation of surrogacy by additionally accounting for the (unobserved) effect of the next action on the future states. Specifically their estimator is of the following form:
\begin{eqnarray*}
\hat V_{\text{LOPE}}^{\pi_e} = \frac{1}{|\mathcal D_b|} \sum_{i=1}^{|\mathcal D_b|} \frac{p(\tau^{(i)}|\pi_e)}{p(\tau^{(i)} | \pi_b)} (G^{(i)} - \hat f(\tau^{(i)}, a_{h+1}^{(i)}) \\ + \underbrace{\mathbb E_{a_{h+1}' \sim \pi_e(.|\tau^{(i)})} \hat f(\tau^{(i)}, a_{h+1}'))}_{\text{useful if surrogacy breaks}},
\end{eqnarray*} where $\hat f$ is the empirical minimizer of $\mathcal D_b$ defined as: $\frac{1}{|\mathcal D_b|} \sum_{i=1}^{|\mathcal D_b|} \left(G^{(i)} - \hat f(\tau^{(i)}, a_{h+1}^{(i)})\right)^2$. They build robustness to the violation of the surrogacy assumption via the second term. However, the LOPE estimator is originally designed for contextual bandit policies, and thus only considers the (unobserved) effect of the immediate next action $a_{h+1}$ on the future states, rather than the full-horizon sequential RL policies.

\paragraph{Online-model based prediction}
We fit a transition model under the new target policy, $\hat p(s_{t+1}|s_t, a_t)$ using the on-policy data $\mathcal D_e$, and first predict the future (unobserved) states under the target policy, then predict the corresponding long-term rewards for those states. ~\citet{tran2024icml} considered settings where with sufficient coverage of the states and actions in the short horizon data, the transition dynamics can be fully learned from the on-policy data alone. In such cases, the online-model based prediction can provide accurate estimates.

\paragraph{Extrapolation of the average short-term reward} We multiply the average reward observed within the short horizon by the full horizon length (similar to the naive baseline in ~\citet{tran2024icml}). This may be reasonable if the state sequences are periodic, and the full cycle is captured in the short-horizon data.

\paragraph{Extrapolation of the last short-term reward} This is a slightly different version of (3) where instead of taking the average value, the last observed reward is multiplied by the full horizon length. This may be reasonable if the states reach an equilibrium or only experience small deviations from the last state after the initial observation window.

\paragraph{Full-horizon Monte Carlo sampling} This requires observing the full horizon outcome of the target policy, so it would not be feasible in settings of our interest where the long-term outcome takes a prohibitively long amount of time to observe. We include this to benchmark the performance of using a finite set of on-policy trajectories to estimate the long-term policy value. Even though MC estimates are unbiased, their variance may be high if the on-policy data is limited, and using the short-horizon soft surrogates may give smaller prediction errors due to the variance-bias trade-off.

\section{Clinical Simulations}
\label{appendix:clinical simulators}
\subsection{HIV Treatment}
\paragraph{Domain} We generate 500 initial patient states with a perturbation rate of 0.05 using the simulator designed by ~\citet{ernst2005}. The environment has no transition stochasticity, so if the policy is deterministic, the generated trajectories from a given initial state and a policy are deterministic. We execute the target policy with perfect fidelity ($\epsilon = 0$) for short horizons, collecting 500 short-term trajectories; while the behavioral policy is executed with 0.95 fidelity, for 5 times per initial patient state, which yields 2500 long-term trajectories. We do not compare our estimators to MC sampling in HIV because due to the environment's lack of stochasticity, we define the ground truth policy value as the average MC estimates.

Reward is given at every time step in the 200 horizon. It is determined by the number of free virus particles and the HIV-specific cytotoxic T cells. Following RL standards, the return $G$ is the un-discounted sum of rewards over the full horizon. 

% with a perturbation rate of 0.05, and the same set of initial states is used for collecting trajectories under both $\pi_b$ and $\pi_2$. The initial states are shared between the behavioral and the new target decision processes but the subsequent state trajectories differ due to the policies taking different actions.

\paragraph{Policy}We first learn an optimal behavioral policy using Fitted Q Iteration (FQI). FQI is implemented with an extra-trees regressor from `sklearn` packages with n estimators = 50 and min samples split = 2. An $\epsilon$-greedy policy ($\epsilon = 0.15$, except when the policy is first initialized, $\epsilon = 1$) is used for collecting on-policy samples under the newly fitted policy, and the policy is rolled out 30 times between policy improvement. 400 Bellman backups are done with a discount factor of 0.98, and we apply 10 iterations of policy improvement, each time using the replay buffer of all previously collected samples. T We set $\pi_b$ to be the optimal policy from the last policy improvement step.

To model scenarios where a new policy is proposed by swapping some of the actions in the behavioral policy with the new (improved) ones, we replace the small dosage actions in the behavioral policy with zero dosages. For example, this means 0.1 of RTI and 0.05 of PI in the behavioral policy is now zero dosage under the new target policy. We assume that the target policy is executed with perfect fidelity, so $\epsilon$ probability of taking a random (sub-optimal) action is 0 under the new policy. As stated before, the historical data has 2500 long-term patient trajectories, and the on-policy data only has 500 short-term patient trajectories and unknown long-term outcomes. 

% The target policy $\pi_2$ replaces small dosages (i.e., 0.1 of RTI and 0.05 of PI) with zero dosages, and is executed with perfect fidelity. For example, for a given state, if $\pi_b$ chooses high RTI and small PI (0.7 RTI and 0.05 PI), then $\pi_2$ chooses high RTI only (0.7 RTI). The historical dataset has 2500 observations of the 500 patients each with five roll-outs, and the target dataset contains 500 trajectories, one per each patient. 
% $\mathcal D_b$ and $\mathcal D_2$ are available as `hiv behavioral' and `hiv target' pickle files.

We've included our code in the supplementary materials and the datasets and the codebase will be shared as a public repo after the anonymous review process.

\paragraph{Soft surrogate estimator training details} Density ratio estimators are implemented as binary classifiers with an MLP classifier from `sklearn' packages. Specifically, we set the labels for the off-policy trajectories as 0 and the labels for the on-policy trajectories as 1. Hyperparameter selection is done with 5-fold cross validation to select the best model parameters over the following sets of values:
\begin{itemize}
\item MLP hidden layer sizes: (128), (128, 64), (128, 64, 64) 
\item MLP $\alpha$ (l2 regularization): 0.001, 0.01, 0.1
\item Learning rate: adaptive, constant
\end{itemize}All models are implemented with the  Adam optimizer and ReLU activation between the layers. 
The un/weighted regression model is implemented with Support Vector Regressor (SVR module from `sklearn'), and we also did 5-fold cross validation for hyperparameter selection over the following values:
\begin{itemize}
\item C (regularization): 0.1, 1, 10
\item Epsilon: 0.001, 0.01, 0.1
\end{itemize}For $k$-fold cross fitting of the doubly robust methods~\citep{chernozhukov2017doubledebiasedmachinelearningtreatment}, we set $k=2$.

\paragraph{Baseline methods training details}\label{appendix:baseline_experiments}We used the same model class and followed the same hyper-parameter selection process for the LOPE estimator ~\cite{saito2024long}, which requires training the density ratio estimator and the regression model. For the regression model, which takes as inputs the short-term trajectories and the last action, we one-hot-encoded the discrete actions and concatenated the trajectory data with the actions.

For training the online model-based RL, we implement the transition dynamics model with `sklearn` packages Multi Output regressor (and SVR as a sub-module). Similarly as before, we performed 5-fold cross validation over the following parameters: 
\begin{itemize}
\item C (regularization): 0.1, 1, 10
\item Epsilon: 0.1, 0.2.
\end{itemize}

The reward extrapolation baselines and MC sampling only require on-policy data, but no model training is required using the historical data.

% Among the baselines considered only the online model-based RL method requires training a model. The two reward extrapolation baselines only depend on the on-policy short-horizon rewards observed until $h$. The online model-based RL requires training a transition dynamics model using the on-policy data in $\mathcal D_2$. The transition model is implemented with Multi output regressor and Support Vector Regressor as a sub-module. Hyperparameter selection with 5-fold cross validation is done to select the best hyperparameters for the transition model. The hyperparameters considered are: \begin{itemize}
% \item C (regularization): 0.1, 1, 10
% \item Epsilon: 0.1, 0.2.
% \end{itemize}The baseline methods are included in `baseline hiv.py'. 
% Due to file size limits, HIV and sepsis data files are unavailable for uploading. These datasets will be shared via a link after the anonymous review process.

\subsection{Sepsis}

\paragraph{Domain} We use the implementation by ~\citet{oberst2019counterfactual}, which defines the patient state as a vector of 5 values: an indicator for whether the patient is diabetic, and four vital signs--heart rate, blood pressure, oxygen concentration, glucose level--that take values in a subset of \{very high, high, normal, low, very low\}. The state space size is 1440. The action space sizes differ between the behavioral policy and the target policy. The behavioral policy chooses from two binary options of antibiotics and mechanical ventilation, thus $|\mathcal A| = 4$; while the target policy has an additional option of vasopressors, which now makes $|\mathcal A^+| = 8$. 

Each patient's trajectory is at most 20 steps long, but may terminate early with -1 if at least three vitals are out of the normal range, or with 1 if all vital signs are in the normal range without any active treatment. If the patient is either dead or discovered, no further reward is given, and all intermediate states are rewarded by 0. We use a discount factor of 0.99 for calculating the long-term policy value. 
% The datasets are shared as `sepsis behavioral traj/return and `sepsis target traj/return' pickle files. 

\paragraph{Policy}Similar to~\citet{namkoong2020off}, we use policy iteration to learn the optimal policy, and create a soft-optimal policy as the behavior policy by having the policy take a random action with probability 0.15. The value function is computed using value iteration with the discount factor of $\gamma=0.99$. Similar process is used to learn the target policy but with the different action space $\mathcal A^+$, which includes includes vasopressors in addition to $\{\text{antibiotics}, \text{ mechanical ventilation}\}$ already available in $\mathcal A$.

In sepsis, the decision process is stochastic, so we define the true policy value as the average MC estimates of 5000 roll-outs. In the on-policy dataset used to estimate the long-term value, we only include 500 bootstrapped samples, so the historical dataset has 5000 full-horizon trajectories and the on-policy dataset only has 500 short-term samples. We evaluate our estimators across 5 random seed runs to compute the mean and the standard deviation of the prediction mean squared errors. In each seed run, the on-policy dataset has a different set of bootstrapped samples. 

\paragraph{Soft surrogate estimator training details}
Since the Sepsis domain has a reasonably sized state space of 1440 discrete states, we build the density ratio estimator $\hat h$ based on the frequency of a given trajectory in the on-policy data versus the off-policy data.

For fitting the regression model, we consider the following sets of hyperparameters to select the best performing hyperparameter from 5-fold cross validation are: \begin{itemize}
\item MLP hidden layer sizes: (128), (128, 64)
\item Learning rate: 0.001, 0.01, 0.1
\end{itemize}All models are implemented with Pytorch MLP Regressor and ReLU activation between the layers.

\paragraph{Baseline methods training details}\label{appendix:baseline_experiments2}In the sepsis simulator, if the patient is dead or discharged by $h$, their return is fully known, so we consider online baselines where if the patient's outcome is known by $h$, then there's no error; and for those whose outcomes are still unknown, their long-term returns are projected to be 0. Note that in the sparse reward setting like sepsis, the two reward extrapolation methods -- `final short-term reward' and `average short-term reward' -- behave the same since any unobserved patient outcomes (both average and last rewards observed so far are 0) are predicted to be 0 even in the future.

Since the states are discrete with size 1440 and the actions are of size 8, we build an online dynamics model based on the number of times a given tuple $(s, a, s')$ is visited from $(s, a)$, $\hat p(s, a, s') = \frac{N(s, a, s')}{N(s, a)}$. For any pairs of $(s, a)$ that are unobserved, we set the transition probabilities to the next state as uniform across the entire state space.  

For fitting the regression part of the SLOPE estimator~\cite{saito2024long}, we consider the following hyperparameters:
\begin{itemize}
\item MLP hidden layer sizes: (128), (128, 64)
\item Learning rates: 0.001, 0.01, 0.1
\end{itemize}All models are implemented with Pytorch MLP Regressor and ReLU activation is applied between the NN layers. We one-hot-encode the actions and concatenate them with the short-term data to construct the inputs to the regression model $f(\tau_h,  a)$. 

\subsection{Computational Resources} Computation for all of the experiments was done on internal servers with GPUs, but they can be done without GPU with a reasonable wall-clock time of less than 10 hours, which may vary depending on the size of the hyperparameter search space and the number of folds for cross validation.

\subsection{Additional Experiments}\label{appendix:additional_results}

\subsubsection{Hypothesis testing of the estimated difference between behavioral policy and target policy's value} While directly predicting the long-term policy value is our primary interest, we also consider the potential application of using our methods to quickly identify whether the new target policy's value is going to be significantly different from the behavioral policy value. To handle this task, we propose running hypothesis testing to compare the observed long-term returns under the behavioral policy and the estimated long-term outcomes of the new target policy. In the HIV domain, the target policy and the behavioral policy are both executed from the same set of 500 initial patients, so we apply paired T-test; while in Sepsis, they share the same initial distribution of the patients but the exact starting conditions may not match, so we perform independent T-test. We use \texttt{scipy.stats} packages \texttt{stats.ttest.rel} for HIV and \texttt{stats.ttest.ind} for Sepsis. In both domains, the behavioral policy value is significantly lower than the target policy. We are interested in whether our estimators can successfully reject the null hypothesis with limited on-policy data from short horizons. 

% The p-values for comparing the target policy's value with the behavioral policy's value are calculated using paired T tests for HIV (implemented with `stats.ttest rel' from `scipy stats` packages) and independent T tests for Sepsis (implemented with `stats.ttest ind'). In HIV, the 500 estimated returns, one per starting patient state, under $\pi_2$ are compared to the 500 averaged returns under $\pi_b$ since the behavioral policy is rolled out with $\epsilon = 0.05$ from each starting state 5 times. In Sepsis, 500 estimated returns under $\pi_2$ are compared to the 5000 observed returns under $\pi_b$ to evaluate whether the true mean under the target policy is likely to exceed the true mean under the behavioral policy. The experiment results are shown in Table 1 of the main text.

% In the main text we include the estimation results from the 10\% of the full horizon ($h = 20$ for HIV and $h = 2$ for sepsis). Here we show the results from different short-horizon lengths. The full horizon length of HIV is 200 time steps, and in sepsis the full horizon is 20 steps.

Table \ref{appendix:t_tests} shows the evaluation of our estimators at different short horizons and their corresponding p-values from running the T-test to compare the estimated target policy's returns and the observed behavioral policy's outcomes. We included the columns for $h = 20$ and $h = 2$ in the main text for the HIV and the Sepsis domains. Promisingly, our results show that in the Sepsis domain, all our estimators achieve low MSE and significant p-values to successfully reject the null hypothesis using the short-horizon data from $h = 2$. In the HIV domain, our estimators, particularly the soft surrogate estimator and the DR soft surrogate estimator, require a longer $h$ to be able to reject the null. Our weighted regression models can still yield significant p-values with much shorter horizon of $h = 5$, which is 2.5\% of the full horizon in the HIV simulations.

Interestingly we observe that the regression-based estimator's performance does not improve monotonically as $h$ increases while the methods that rely more heavily on on-policy data improve monotonically with a larger $h$. We suspect this is due to the increased covariate shift between the training and the test distribution. While increased $h$ gives more signal about the policy's effects on the states and the rewards, the regression models are also trained on inputs of longer lengths, $\tau_{0:h}$, which may affect the regression model's variance if the training data is limited. We leave it as an interesting open question to explore the selection of short horizon $h$ to handle the trade-off between more signal in the data and higher variance in the estimators.

 \begin{table*}[t!]
\centering
\caption{\textbf{HIV estimation results at different short-horizons: MSE and p-value from pair T-test.}}
\begin{tabular}{@{}lcc|cc|cc|cc@{}}
\toprule 
 & \multicolumn{2}{c}{h = 5} & \multicolumn{2}{c}{h = 10} & \multicolumn{2}{c}{h = 20} & \multicolumn{2}{c}{h = 50}  \\ 
Estimator & Error & p-value & Error & p-value &Error  & p-value & Error  & p-value\\ \midrule
Soft surrogate estimator    & 68.16 & p $\geq$ 0.5    & 68.15 &  p $\geq 0.5$ & 67.15 & p $\geq 0.05$  & 57.96 & p $\leq 10^{-6}$\\
Weighted soft surrogate estimator   & 61.84  & p $\leq$ $10^{-6}$   & 71.51  &  p $\leq$ $10^{-6}$ & 53.21& p $\leq 10^{-6}$ & 53.85 & p $\leq 10^{-6}$\\
DR soft surrogate   & 67.99  & p $\geq$ 0.7   & 68.32  & p $\geq 0.3$ & 64.86&  p $\leq 0.01$& 58.15 & p $\leq 10^{-6}$\\ 
DR weighted soft surrogate   & 57.80 & p $\leq 10^{-6}$ & 70.90  & p $\leq$ $10^{-6}$ & 50.61 & p $\leq 10^{-6}$ & 55.26 & p $\leq 10^{-6}$ \\
LOPE~\cite{saito2024long} & 73.84 & p $\leq 10^{-6}$&69.95 & p $\leq 10^{-5}$ & 67.75 & p $\geq 0.3$ & 71.91 &p $\leq 10^{-6}$ \\
\bottomrule
\end{tabular}

\caption{\textbf{Sepsis estimation results at different short-horizons: MSE (mean $\pm$ std across 5 seeds) and p-value from pair T-test.}}
\begin{tabular}{@{}lcc|cc|cc@{}}
\toprule 
 & \multicolumn{2}{c}{h = 2} & \multicolumn{2}{c}{h = 4} & \multicolumn{2}{c}{h = 5} \\ 
Estimator & Error & p-value & Error & p-value &  Error & p-value\\ \midrule
Soft surrogate estimator    & 0.04 (0.02) & p $\leq 10^{-6}$   & 0.02 (0.01) &  p $\leq 10^{-6}$ & 0.02 (0.02) & p $\leq 10^{-6}$ \\
Weighted soft surrogate estimator   & 0.03 (0.01) & p $\leq 10^{-6}$ & 0.01 (0.01)  & p $\leq$ $10^{-6}$   & 0.04 (0.01) & p $\leq 10^{-6}$ \\
DR soft surrogate   & 0.04 (0.01) & p $\leq 10^{-6}$   & 0.01 (0.005)  & p $\leq 10^{-6}$ &  0.03 (0.01) & p $\leq 10^{-6}$\\ 
DR weighted soft surrogate   & 0.03 (0.01)  & p $\leq 10^{-6}$ & 0.02 (0.004) & p $\leq$ $10^{-6}$ & 0.01 (0.01) & p $\leq 10^{-6}$ \\
LOPE~\cite{saito2024long} & 0.14 (0.25) & p $\leq 10^{-6}$ & 0.10 (0.01) & p $\leq 10^{-6}$ & 0.19 (0.13) &p $\leq 10^{-6}$ \\
\bottomrule
\end{tabular}
\label{appendix:t_tests}
\end{table*}

\subsubsection{Evaluation across different target policy values}
We considered a wider range of target policy values to evaluate the performance of our proposed estimators against the baselines when the target policy's value differs more or less than the behavioral policy's value. 

\begin{table}[t!]
\centering
\caption{\textbf{Each column represents a unique target policy, whose value is listed on the first row. The rest of the table shows the MSE made by the estimators for each of the 8 target policies in the HIV domain. All estimates are made with the short-term data from $h=20$. The first 4 estimators are our proposals based on regression.}}
\resizebox{\linewidth}{!}{
\begin{tabular}{ c |c| c| c| c| c| c| c| c }
\toprule 
 
\textbf{Estimator}    \\ \midrule
Ground-truth policy value & 350.55 &417.37 &419.94 & 404.87 & 342.49 & 411.18 & 407.44 & 401.89 \\ \midrule
Weighed soft surrogate estimator & 0.61 & 66.78 & 68.57 & 53.21 & 8.68 & 62.16 & 56.07 & 51.6 \\ 
Soft surrogate estimator & 12.83 & 79.66 & 82.22  & 67.15 & 6.41 & 73.84 & 69.73 & 64.7 \\  
DR soft surrogate  & 10.49 & 77.19 & 79.83 & 64.86 & 8.93 & 66.94 & 67.38 & 65.53 \\
DR weighted soft surrogate & 2.35 & 69.14 & 71.07 & 50.61 & 2.63& 64.41 & 57.87 & 51.69 \\
LOPE~\cite{saito2024long}& 13.43 & 80.26 & 82.83 & 67.75 &  5.38 & 74.06 & 70.33 & 64.78\\
Online model-based method & 348.19 & 415.15 & 417.68 & 402.49 & 340.15 & 408.98 & 405.18 & 399.52\\
Last reward extrapolation &  349.51 & 416.32 & 415.5 & 400.56 & 338.89 & 407.47 & 403 & 398.06\\
Average reward extrapolation & 346.28 & 413.07 & 418.86 & 403.83 & 341.59 & 410.23 & 406.36 & 400.88\\
\bottomrule
\end{tabular}}
\label{supp_table:hiv}
\end{table}

In the HIV experiments, we considered the case of the target policy replacing the behavioral policy's small dosage with zero and executed with perfect fidelity. We now introduce target policies where the small dosage is reduced by 25\%, 50\%, 75\%, and 100\% of the original amount, and each policy is executed with different levels of fidelity, i.e., we consider $\epsilon \in \{0, 0.1\}$, which means with a probability of $\epsilon$, the policy chooses a random action instead of the optimal action prescribed by the policy. The resulting target policy values are: 350.55, 417.37, 419.94, 404.87, 342.49, 411.18, 407.44, 410.89. The behavioral policy value is 337.01. We show the performance of different estimators--including our own and the baselines-- across these 8 target policies in Table \ref{supp_table:hiv}. In the HIV domain, we observe that our estimators outperform all the baselines across all target policy values.

In the Sepsis domain, we considered new target policies with varying values of epsilon ($\epsilon$ in \{0.01, 0.02, 0.05, 0.1, 0.2, 0.25, 0.3\}). The target policy, whose results are reported in the main text, is executed with $\epsilon = 0.15$. With different values of epsilon, we get the following target policy values: 0.119, 0.110, 0.057, 0.019,-0.006, -0.022, -0.053, -0.065. The behavioral policy value is -0.178. We show the performance of our estimators across different policy values in Table \ref{supp_table:sepsis}.

\begin{table}
\centering
\caption{\textbf{Each column represents a unique target policy, whose value is listed on the first row. The rest of the table shows the MSE made by the estimators for each of the 8 target policies in the Sepsis domain. All estimates are made from the short-term data from $h = 2$. The first four estimators are our own based on regression.}}
\resizebox{0.9\linewidth}{!}{
\begin{tabular}{ c| c |c| c |c| c| c |c| c }

\toprule 
\textbf{Estimator}\\ \midrule
Ground-truth policy value & 0.119 & 0.110 & 0.057 & 0.019 & -0.006 & -0.002 & -0.053 & -0.065 \\ \midrule

Weighted soft surrogate estimator & 0.15 & 0.10 & 0.08 & 0.06 & 0.03 & 0.04 & 0.02 & 0.03 \\ 
Soft surrogate estimator & 0.14 & 0.11 & 0.06 & 0.04 & 0.04 & 0.02 & 0.008 & 0.02 \\  
DR soft surrogate & 0.13 & 0.11 & 0.07 & 0.04 & 0.04 & 0.03 & 0.01 & 0.01\\
DR weighted soft surrogate & 0.13 & 0.10 & 0.10 & 0.08 & 0.03 & 0.02 & 0.01 & 0.03\\
LOPE~\cite{saito2024long} & 0.39 & 0.09 & 0.05 & 0.07 & 0.14 & 0.06 & 0.06 & 0.16\\
Online model-based method & 0.37 & 0.39 & 0.48 & 0.55 & 0.62 & 0.62 & 0.67 & 0.67 \\
Last reward extrapolation &  0.002 & 0.005 & 0.05 & 0.10 & 0.07 & 0.07 & 0.1 & 0.06\\
Average reward extrapolation & 0.002 & 0.005 & 0.05 & 0.10 & 0.07 & 0.07 & 0.1 & 0.06\\ \bottomrule
\end{tabular}}
\label{supp_table:sepsis}
\end{table}

Surprisingly, we observe that the short-term reward extrapolation baselines perform well on 2 out of the 8 target policies. In particular, these policies have small epsilon values of 0.01 and 0.02. When the epsilon is small, most of the patients, who do not die or recover within the short horizon of $h = 2$, receive a reward of 0. Therefore, predicting the unknown patient outcomes as 0, as done by the short-term reward baselines, actually yields good estimates close to the ground-truth values. However, these baselines are brittle to the environment's stochasticity as we observe that with higher epsilon values, most patients' outcomes are unknown from the short horizon. As the target (ground-truth) policy value deviates away from 0, predicting the unknown patient outcome as 0 becomes a poor estimate. For example, when the epsilon is set to 0.25 or 0.3, the reward extrapolation baselines incur the large error of 0.1 and 0.06 compared to the estimator's MSE of 0.01 and 0.02. 

 Our experiments in the main text compare the performance of the estimators across different values of $h$, but the selection of the short-horizon value remains an interesting open question in our work. Especially in real-life scenarios, where decision-makers have leverage in deciding how long to run experiments for, the selection of $h$ for the desired level of prediction accuracy is important to discuss for practical considerations. Potential trade-offs between more signal from the on-policy data versus increased observation costs as well as the increased covariate shift between the training and the test data may factor into this decision.
\clearpage
\bibliography{reference}

\end{document}